\documentclass{article}

\usepackage[main, final]{neurips_2026}

\usepackage[utf8]{inputenc} %
\usepackage[T1]{fontenc}    %
\usepackage{hyperref}       %
\hypersetup{hidelinks}
\usepackage{url}            %
\usepackage{booktabs}       %
\usepackage{amsfonts}       %
\usepackage{nicefrac}       %
\usepackage{microtype}      %
\usepackage{xcolor}         %

\usepackage{graphicx}
\usepackage{makecell}
\usepackage{amsmath}
\usepackage{amssymb}
\usepackage{mathtools}
\usepackage{amsthm}
\usepackage{aliascnt}
\usepackage{enumitem}
\usepackage{subcaption}
\usepackage{url}
\usepackage[capitalize]{cleveref}
\usepackage{xspace}
\usepackage{longtable}      %
\usepackage{algorithm}
\usepackage{algpseudocode}
\usepackage{wrapfig}

\crefname{theorem}{Theorem}{Theorems}
\crefname{proposition}{Proposition}{Propositions}
\crefname{lemma}{Lemma}{Lemmas}
\crefname{corollary}{Corollary}{Corollaries}
\crefname{definition}{Definition}{Definitions}
\crefname{assumption}{Assumption}{Assumptions}
\crefname{remark}{Remark}{Remarks}
\crefname{example}{Example}{Examples}
\Crefname{theorem}{Theorem}{Theorems}
\Crefname{proposition}{Proposition}{Propositions}
\Crefname{lemma}{Lemma}{Lemmas}
\Crefname{corollary}{Corollary}{Corollaries}
\Crefname{definition}{Definition}{Definitions}
\Crefname{assumption}{Assumption}{Assumptions}
\Crefname{remark}{Remark}{Remarks}
\Crefname{example}{Example}{Examples}
\crefname{appendix}{Appendix}{Appendices}
\Crefname{appendix}{Appendix}{Appendices}

\crefname{equation}{Equation}{Equations}
\Crefname{equation}{Equation}{Equations}

\theoremstyle{plain}

\newaliascnt{proposition}{theorem}
\newtheorem{proposition}[proposition]{Proposition}
\aliascntresetthe{proposition}
\newtheorem*{proposition*}{Proposition}

\newaliascnt{lemma}{theorem}
\newtheorem{lemma}[lemma]{Lemma}
\aliascntresetthe{lemma}
\newtheorem*{lemma*}{Lemma}

\newaliascnt{corollary}{theorem}
\newtheorem{corollary}[corollary]{Corollary}
\aliascntresetthe{corollary}
\newtheorem*{corollary*}{Corollary}

\theoremstyle{definition}
\newaliascnt{definition}{theorem}

\aliascntresetthe{definition}
\newtheorem*{definition*}{Definition}

\newaliascnt{assumption}{theorem}
\newtheorem{assumption}[assumption]{Assumption}
\aliascntresetthe{assumption}
\newtheorem*{assumption*}{Assumption}

\theoremstyle{remark}
\newaliascnt{remark}{theorem}

\aliascntresetthe{remark}
\newtheorem*{remark*}{Remark}

\newaliascnt{example}{theorem}

\aliascntresetthe{example}
\newtheorem*{example*}{Example}

\newcommand{\ours}{\texttt{SigCTO}\xspace}
\newcommand{\ourslibrary}{\texttt{SigCTO}\textsuperscript{\textdagger}\xspace}

\newcommand{\intervalj}{I_j}

\definecolor{My_NavyBlue}{rgb}{0.0, 0.0, 0.5}

\newcommand{\defined}[1]{{\color{black}\text{#1}}}

\definecolor{DarkRed}{rgb}{0.7, 0.1, 0.1}

\definecolor{DarkGreen}{rgb}{0.1, 0.5, 0.1}

\newcommand{\dif}{\mathrm{d}}

\usepackage{amsmath,amsfonts,bm}

\def\eqref#1{equation~\ref{#1}}

\def\1{\bm{1}}

\DeclareMathAlphabet{\mathsfit}{\encodingdefault}{\sfdefault}{m}{sl}
\SetMathAlphabet{\mathsfit}{bold}{\encodingdefault}{\sfdefault}{bx}{n}

\newcommand{\E}{\mathbb{E}}

\DeclareMathOperator*{\argmin}{arg\,min}

\title{Conservative Continuous-Time\\ Treatment Optimization}

\author{%
  Nora Schneider\,\thanks{Equal Contribution}\,\,\,$^{ 1 \, 2 \, 3}$\\
  \texttt{nora.schneider@tum.de}
  \And 
  Georg Manten\,\footnotemark[1]\,\,\,$^{1 \, 2 \, 3}$\\
  \texttt{georg.manten@tum.de}
  \And
  Niki Kilbertus\,$^{1 \, 2 \, 3}$ \\
  \texttt{niki.kilbertus@tum.de}
  \AND
  \\[-0.7em]
    $^{1}$ Technical University of Munich \quad $^{2}$ Helmholtz Munich \\
    $^{3}$ Munich Center for Machine Learning
}

\begin{document}

\maketitle

\begin{abstract}
    Finding optimal treatment plans in continuous time is crucial for personalized medicine.
    Existing methods, however, do not explicitly optimize treatment plans and instead focus solely on predicting treatment effects.
    Naive optimization over these treatment-response models can exploit model errors, resulting in plans that appear optimal under the learned model but perform poorly in practice.
    To mitigate model exploitation, we propose \ours (Signature-based Conservative Treatment Optimization), a conservative framework for continuous-time treatment optimization. 
    We derive a path-space regularizer based on signature-kernel maximum mean discrepancy that penalizes treatment plans for which predicted and true responses differ.
    Under suitable assumptions, the regularizer allows for consistent estimation, and the regularized objective upper bounds the true treatment cost.
    Experiments on pharmacometric simulators show that \ours finds treatment plans with lower ground-truth costs than state-of-the-art baselines.
    On real-world ICU data, \ours predicts treatment cost more accurately than baselines, and the regularizer detects prediction errors.
\end{abstract}

\section{Introduction}
A key challenge in healthcare is designing optimal treatment plans that steer a patient’s health state toward a desired clinical outcome, e.g., finding a chemotherapy dosing schedule that maximally reduces tumor volume. 
This task is inherently temporal: patient states evolve continuously, while a treatment plan specifies not only what kind of therapy to apply, but also specific dosing schedules.
In such high-stakes settings, direct experimentation with alternative treatment plans is often infeasible or unethical.
As a result, treatment optimization requires models that predict patient responses under candidate plans and enable optimization over unseen treatment plans.

Historically, treatment optimization has relied on mechanistic, pharmacometric models to describe treatment-response dynamics and inform treatment decisions \citep{minichmayr2024model,sheiner2000pharmacokinetic}.
These manually specified, hard-earned models are interpretable and formulate treatment-response dynamics in continuous time.
Despite these strengths, they often fail to capture the complex, stochastic dynamics of heterogeneous patient populations.
The limited flexibility has motivated data-driven causal inference methods that use longitudinal observational data to estimate patient responses under alternative treatment plans \citep{bica2020estimating,robins2000marginal}.

Most causal methods operate in discrete time and impose a fixed, regularly spaced time grid.
Such a grid, however, does not match clinical data where measurements are typically made at irregular time points \citep{callahan2020research}. 
Moreover, candidate plans are restricted to the grid, even though clinical treatment schedules are often irregular and patient-specific.
More fundamentally, the grid discretizes continuously evolving treatment-response dynamics and can thereby compromise causal inference \citep{SunCrawford2023discretization}.
Without imposing discretization grids, continuous-time methods naturally accommodate irregularly timed measurements and interventions \citep{sun2022causal}.
Recent causal inference methods have embraced this paradigm by modeling response dynamics with learned differential equations \citep{hess2025stabilized,ying2025causal,hess2023bayesian,seedat2022continuous}.
Yet these causal methods stop at treatment response prediction -- treatment optimization in continuous time remains largely unaddressed in this literature.

One key challenge of data-driven treatment optimization is that it must be conducted offline: we aim to find improved treatment plans using only historical data, without actively collecting new data under explorative alternative plans.
Such novel plans may steer the system towards regions with limited historical data support, where predictive models are often inaccurate.
Optimizing a predictive model can therefore exploit model errors, finding spuriously good-looking plans.
This problem is well known in offline reinforcement learning \citep{yu2020mopo,kidambi2020morel} and relates in causal inference to violations of the commonly made overlap assumption, which is already fragile in static settings \citep{d2021overlap}.
In continuous time, treatment plans live in an infinite-dimensional path space, and the overlap assumption becomes practically impossible to satisfy -- making model exploitation a central challenge for reliable continuous-time treatment optimization.

In this work, we propose \ours, a framework for conservative, offline optimization of open-loop treatment plans in continuous time.
We model patient-response dynamics with a controlled stochastic differential equation where treatment is the control input.
To mitigate model exploitation, we augment the predicted cost with a path-space regularizer that penalizes differences between predicted and true responses.
We show that, for any candidate treatment plan, the conservative objective upper-bounds the true cost and discourages plans that exploit overly optimistic predictions.
We estimate the regularizer from data using a consistent, signature-kernel MMD estimator.
Experiments on pharmacometric simulators show that \ours mitigates model exploitation and finds treatment plans with lower ground-truth costs than conservative and non-conservative baseline methods. 
On real-world ICU data, the regularizer detects the distribution mismatch it is designed to penalize.
\section{Related Work}
\textbf{Continuous-time longitudinal causal inference} estimates the potential outcome trajectory under a candidate treatment plan and does not address finding an optimal treatment plan. Some methods estimate population-level potential outcomes \citep{ying2025causal,hizli2023causal,schulam2017reliable,lok2008statistical} and leave patient-specific outcomes unaddressed. 
Other methods predict personalized outcomes by directly conditioning on patient history using neural ODEs \citep{de2022predicting} or neural CDEs \citep{seedat2022continuous, hess2023bayesian}, but only consider binary treatments applied at discrete event times.
INSITE \citep{kacprzyk2024ode} allows for continuous-valued treatments by modeling them as control inputs in a governing ODE, but is also restricted to outcome prediction.

Treatment optimization is not a simple extension of outcome-prediction methods for two reasons. 
First, these methods do not specify nor readily support structured optimization, and a brute-force search over the full treatment plan space is often infeasible.
When the dynamics are known and correctly specified, classical stochastic optimal control \cite{bank2025stochastic,berret2020efficient} provides structured optimization methods.
Second, unconstrained search can, and typically will, exploit model errors. Thus, unconstrained search would only be reliable if the true dynamics were perfectly learned. This can only work in the infinite-data regime and if the dynamics are identifiable from observations in the assumed model class.
Identifiability in turn requires the overlap assumption to be satisfied \cite{lok2008statistical,robins2008estimation,ying2025causal,sun2022causal}.
Intuitively, overlap requires that for all observed patient histories, each treatment plan must be observed with a nonzero probability.
Even for static settings, the overlap assumption is a strong one and often compromised \citep{d2021overlap}.
In continuous time, observational data must support entire treatment paths across times and patient histories, making overlap even harder to satisfy. 
Even when one is willing to buy into the assumptions required for identifiability to hold, practical remedies, like \ours, are required for reliable plans.
Thus, optimizing treatment plans directly with prediction-only methods may select treatment plans outside the data support, where identifiability guarantees do not hold, and predictions are unreliable.

\textbf{Dynamic Treatment Regimes (DTR)} methods learn a sequence of history-dependent decision rules that map patient information at each decision point to a treatment to optimize clinical outcomes \citep{murphy2003optimal,robins2004optimal,chakraborty2014dynamic,li2023optimal,tsirtsis2023finding}.
Most of this literature is formulated in discrete time and is often developed with sequential decision making (and sometimes adaptive data collection) in mind rather than personalized planning from a fixed observational dataset.
Offline variants exist -- for example, \citet{zhou2023optimizing} propose a pessimistic Bayesian Q-learning approach for learning DTRs from observational data, and \citet{wang2025variational} propose variational counterfactual intervention planning -- but remain discrete-time frameworks.
As a result, existing DTR methods do not address irregular observations typical in clinical time series, where measurements and treatments occur at non-uniform times.
\citet{xu2024modeling} instead learn a joint model not only over dosage values, but also dosage times, while still fundamentally considering a discrete sequence of events rather than a continuous-time treatment path.

\textbf{Offline Reinforcement Learning (RL)} uses pessimism and uncertainty-aware regularization to mitigate model exploitation. For discrete-time Markov Decision Processes (MDPs), model-based RL approaches propose uncertainty-penalized rewards, conservative value estimates, or uncertainty-based rollout truncation \citep{yu2020mopo,kidambi2020morel,yu2021combo,zhang2023uncertainty,koprulu2025neural}. In continuous time, existing methods are either model-free approaches \citep{zhang2023continuous,burns2023offline} or introduce novel dynamics models and do not explicitly address model errors during optimization \citep{hoischen2025operator,holt2023neural}. Moreover, RL typically optimizes closed-loop control; \ours optimizes open-loop plans and explicitly accounts for model errors.
\section{Problem Statement}\label{subsec:problem}

We formalize the problem of finding optimal treatment plans from irregular observational data following the causal perspective on treatment-effect estimation.

\textbf{Setting.}
Let $I=[t_s,t_f]\subset\mathbb R$ be a compact time interval and let $(\Omega,\mathcal F,\mathbb F,P)$ be a complete filtered probability space with filtration $\mathbb F=(\mathcal F_t)_{t\in I}$. 
On this space, we consider three observable stochastic processes: the \emph{treatment process} $U: I \times \Omega \to A \subseteq \mathbb R^{d_u}$ (e.g., chemotherapy), the \emph{outcome process} $Y:I \times \Omega \to \mathbb{R}^{d_y}$ (e.g., tumor volume), and the \emph{covariate process} $Z:I \times \Omega \to \mathbb{R}^{d_z}$ (e.g., biomarkers).
The treatment is the external intervention, while outcomes and covariates describe the patient’s evolving condition.
For convenience, we combine the outcome and covariate process into the \emph{state process} $X=(Z,Y): I \times \Omega \rightarrow \mathbb R^{d_x}$, where $d_x=d_y+d_z$. 
Following the potential outcome framework in causal inference \cite{rubin2005causal}, for a deterministic treatment plan $u:I\to A\subseteq\mathbb R^{d_u}$, we write $X^{(u,\mathrm{pot})} : I \times \Omega \rightarrow \mathbb R^{d_x}$ for the \emph{potential state process}, i.e., the patient state trajectory that would be observed if treatment were externally set to follow $u$ over the full time interval $I$.

Throughout the paper, we use the following notation. Let $X=(X_t)_{t\in I}$ be a stochastic process indexed by $I$. 
The restriction of $X$ to the interval $[a,b]\subseteq I$ is denoted by $X_{[a,b]}$. Its path law, $P_{X_{[a,b]}}$, is the
probability distribution of the entire trajectory over this
interval.\footnote{$P_{X_{[a,b]}} = P \circ X_{[a,b]}^{-1} \in \mathcal{M}_1((\mathbb{R}^{d_x})^{[a,b]})$, where $\mathcal{M}_1(\bullet)$ is the set of all probability measures on $\bullet$, $(\mathbb{R}^{d_x})^{[a,b]}$ the set of maps from set $[a,b]$ to set $\mathbb{R}^{d_x}$.}
The natural filtration of $X$ is denoted by $\mathcal{F}^X=(\mathcal{F}_t^X)_{t\in I}$, where $\mathcal{F}^{X}_t = \sigma(X_{[t^\prime]}\mid  t_s \leq t^\prime \leq t)$, that is, the sigma-algebra generated by the values of $X$ up to time $t$. We denote the conditional path law of a process $Y$ given $X_{[a,b]}$ by $P_{Y\mid X_{[a,b]}} = P_{Y \mid \mathcal{F}^X_{a, b}}$, where $\mathcal{F}^X_{a, b} = \sigma (X_t \mid a \leq t \leq b)$.
We use lowercase letters to denote sample-path realizations, i.e., $x=(X_t(\omega))_{t\in I}$ for $\omega\in\Omega$.  

\textbf{Data.}
We observe a dataset $\mathcal D=\{(u^{(i)},y^{(i)},z^{(i)})\}_{i=1}^N$
of $N$ i.i.d.\ patient trajectories.
Here, $u^{(i)}$, $y^{(i)}$, and $z^{(i)}$ are the treatment,
outcome, and covariate time series, respectively, recorded at
patient-specific, potentially irregular times in $[t_s,t_f]$. The dataset only captures patient trajectories under historical treatment plans; potential state trajectories under alternative plans are generally unobserved.

\textbf{Objective.} %
Our goal is to find treatment plans that optimize patient outcomes according to a clinical objective. Specifically, at decision time $t_0 \in I$ , given a patient's history available up to $t_0$ (i.e., $\mathcal{F}_{t_0}$), we seek an optimal treatment plan $u^* : [t_0, t_f] \to \mathbb{R}^{d_u}$ with
\begin{align}\label{eq:treatment_optimization}
u^* = \argmin_{u \in \mathcal{U}_{\mathrm{adm}}}
\underbrace{\mathbb{E}\Big[
    \int_{t_0}^{t_f} f\big(t, X^{(u, \mathrm{pot})}_t, u_t\big)\dif t
    + g\big(X^{(u, \mathrm{pot})}_{t_f}\big)
    \Big| \mathcal{F}_{t_0} 
\Big]}_{=:\mathcal{J}(t_0,X^{(u,\mathrm{pot})},u)}.
\end{align}
The class of admissible treatment plans $\mathcal{U}_{\mathrm{adm}}$ consists of deterministic, open-loop treatment plans. The running cost $f:I \times \mathbb R^{d_x}\times A\to\mathbb R$ and final cost $g:\mathbb R^{d_x}\to\mathbb R$ are prespecified and encode the clinical objective. For instance, the running cost $f$ may penalize treatment intensity, whereas the final cost $g$ may represent a target tumor volume at the end of the planning horizon. Intuitively, the state process $X$ describes patient trajectories under historical treatment $U$, whereas $X^{(u, \mathrm{pot})}$ describes the generally unobserved potential trajectory had treatment been set to a deterministic plan $u$. \cref{eq:treatment_optimization} seeks a plan with the lowest expected cost under these potential trajectories, using only observational data $\mathcal{D}$.
\section{Conservative Continuous-Time Treatment Optimization} \label{sec:conservative-continuous-time-treatment-optimization}
In this section, we derive signature-based, conservative treatment optimization (\ours).  
The treatment optimization objective in \cref{eq:treatment_optimization} poses two key difficulties:
First, because the data is observational, the potential state process is generally unobserved \citep{imbens2015causal}. %
We thus introduce additional assumptions under which its law is identifiable from the observational distribution (\cref{sec:structural_dyn_model}). 
Second, we have only access to a finite sample from the observational distribution.
Thus, even when identifiability is given theoretically, estimates of the potential state distribution -- and hence of the resulting cost -- can be inaccurate. Minimizing the predicted cost therefore need not minimize the true cost and may favor treatment plans with overly optimistic cost estimates. We therefore derive a theoretical upper bound on the true cost (\cref{sec:conservative_optimization}) and a consistent estimator of this bound (\cref{sec:sigCTO}).

\subsection{Structural Dynamics Model for Identifiability} \label{sec:structural_dyn_model}

Inspired by \citet{kacprzyk2024ode}, we establish identifiability of the potential state process from the observational distribution by assuming an underlying structural dynamics model.
We use controlled stochastic differential equations (SDEs), which provide a well-established framework for modeling noisy dynamical systems and are particularly well-suited to pharmacology and disease modeling, where diseases and treatment responses evolve continuously and stochastically \citep{geng2017prediction}.

\begin{assumption}[Structural Dynamics Model] \label{assum:mechanistic_model}
The observed state process $X=(X_t)_{t\in I}$ evolves under the, possibly stochastic, treatment process $U=(U_t)_{t\in I}$ according to
\begin{align}\label{eq:governing_sde}
    \mathrm{d}X_t = \mu(X_t,U_t)\,\mathrm{d}t + \sigma(X_t,U_t)\,\mathrm{d}W_t, 
\end{align}
where $X_0\sim\mu_0 \in \mathcal{M}_1(\mathbb{R}^{d_x})$, $W=(W_t)_{t\in I}$ is a standard Wiener process, and $\mu$, $\sigma$ are shared structural drift and diffusion functions, respectively. For a deterministic treatment plan $u:I\to A$, the potential state process $X^{(u,\mathrm{pot})}$ is defined as the solution of \cref{eq:governing_sde} with $U_t=u_t$ for all $t\in I$.
\end{assumption}

We further make standard regularity assumptions on $(\mu,\sigma)$, such as Lipschitz continuity and linear growth, to ensure the existence and path-wise uniqueness of the solution of the controlled SDE (see \cref{app:regularity_assum} for details). In the following, we write $X^{t_0,x_0,u}$ for the strong solution of \cref{eq:governing_sde} on $[t_0,t_f]$ initialized at $X_{t_0}=x_0$ and driven by the deterministic plan $u$ and denote its \textit{solution path law} by $P_{X^{t_0,x_0,u}}$ or $P_{X_{[t_0,t_f]}^{t_0,x_0,u}}$. 
Under \cref{assum:mechanistic_model}, the observed state process is recovered by the potential process under the realized treatment plan, i.e., $X=X^{(U,\mathrm{pot})}$ a.s., similar to the classical consistency assumption in related treatment-effect estimation literature.

\begin{assumption}[Observability and Structural Identifiability]\label{assum:structural_identifiability}
The dynamics $(\mu, \sigma)$ are identifiable from the observational distribution up to response-equivalent parameterizations on the admissible state-control domain.
\end{assumption}

Together, \cref{assum:mechanistic_model,assum:structural_identifiability} require the observational distribution to identify the system's response to every admissible treatment plan. In related treatment-effect estimation literature, this typically corresponds to the omnipresent consistency, overlap, and no-hidden-confounding assumptions.

\begin{proposition}[Structural Identifiability]\label{prop:structural_identifiability}
Under \cref{assum:mechanistic_model,assum:structural_identifiability} and the regularity conditions in \cref{app:admissible_treatments}, the law of $X^{(u,\mathrm{pot})}$ and consequently the expected cost $\mathcal{J}(t_0,X^{(u,\mathrm{pot})},u)$  in \cref{eq:treatment_optimization} are identifiable for every $u\in\mathcal{U}_{\mathrm{adm}}$.
\end{proposition}
A proof is provided in \cref{app:proof_identifiability}. Identifiability means that the population observational distribution uniquely determines the potential state law for each admissible plan. Importantly, identifiability is an idealized, infinite-data statement: it does not ensure accurate estimation from finite data. Moreover, under \cref{assum:mechanistic_model} the dynamics are Markovian and we can rewrite the treatment optimization objective as the following stochastic open-loop control problem
\begin{align}
u^*
:=
\argmin_{u\in\mathcal{U}_{\mathrm{adm}}}
\underbrace{
\mathbb{E}\left[
\int_{t_0}^{t_f}
f(t,X_t^{t_0,x_0,u},u_t)\,\mathrm{d}t
+
g(X_{t_f}^{t_0,x_0,u})
\right]
}_{=:\mathcal{J}(t_0,X^{t_0,x_0,u},u)}.
\label{eq:optimization_rephrased_as_OC}
\end{align}

\subsection{Conservative Treatment Optimization}\label{sec:conservative_optimization}
Even under assumptions that allow for identifiability of the potential state process under a candidate treatment plan, and thus the associated ground-truth cost, finite-data estimates can still be inaccurate. To account for these estimation errors, we derive a tractable upper bound on the ground-truth cost in \cref{eq:optimization_rephrased_as_OC}. \ours minimizes the upper bound, resulting in conservative treatment optimization.

\textbf{Approximated Dynamics.}
To construct a model-based surrogate for the unknown ground-truth cost of a treatment plan, we fit a model for the structural dynamics in
\cref{eq:governing_sde} from observational data using a well-specified model class. Our experiments use neural SDEs \citep{issa2023non},
but the construction also allows other model classes. Given a candidate plan $u$ and initial state $\hat X_{t_0} = x_0$, the learned model induces an estimated state process  $\hat{X}^{t_0,x_0,u}$. Because the cost function is known, substituting the process
gives the \emph{model-based surrogate cost} $\mathcal{J}(t_0, \hat{X}^{t_0,x_0,u}, u)$. Next, we relate the model-based surrogate to the ground-truth cost.

\textbf{Decomposition of Estimation Error.}
We next express the difference between the ground-truth and surrogate costs in terms of local discrepancies between the true and learned dynamics. For any temporal partition of the treatment horizon, the following lemma decomposes this cost difference into a telescoping sum. Each term measures a conditional one-step discrepancy over one subinterval. For any interval $\intervalj:=[t_j,t_{j+1}]$, we denote the true and learned solution path laws initialized at $x$ at time $t_j$ and driven by the deterministic treatment segment $u_{\intervalj}$, respectively, by $P_{\hat{X}_{\intervalj}^{t_j,x,u_{\intervalj}}}$ and $P_{X_{\intervalj}^{t_j,x,u_{\intervalj}}}$.
\begin{lemma}[Estimation Error Telescope]\label{lemma:telescope_lemma}
For an admissible treatment plan $(u: I \to \mathbb{R}^{d_u}) \in \mathcal{U}_{\mathrm{adm}}$, initial state $x_0$, and any temporal partition $t_0 < t_1 < \cdots < t_K = t_f$, 
\begin{align*}
\mathcal{J}(t_0, X^{t_0,x_0,u}, u) - \mathcal{J}(t_0, \hat{X}^{t_0,x_0,u}, u) = \sum_{j=0}^{K-1} \mathbb{E}_{x \sim P_{\hat{X}_{t_j}}} [D_j(x,u)] 
\end{align*}
with one-step discrepancy on $\intervalj$ between true and learned dynamics given by 
\begin{align*}
    D_j(x,u) &= \mathbb{E}_{x_{\intervalj} \sim P_{X_{\intervalj}^{t_j,x,u_{\intervalj}}}} \left[h_j(x_{\intervalj})\right]  -\, \mathbb{E}_{\hat{x}_{\intervalj} \sim P_{\hat{X}_{\intervalj}^{t_j,x,u_{\intervalj}}}} \left[h_j(\hat{x}_{\intervalj})\right]     
\end{align*} 
and $h_j$ is the deterministic path functional
$x_{\intervalj} \mapsto \int_{t_j}^{t_{j+1}}\! f(s, x_{s}, u_s)\dif s + \mathcal{J} (t_{j+1}, x_{t_{j+1}}, u)$ mapping a path segment $x_{\intervalj}$ to a sum of the cumulative cost on $\intervalj$ and the continuation value at $t_{j+1}$.
\end{lemma}
A proof is provided in \cref{app:proof_telescope}, and although it requires different techniques, it can be viewed as an extension of related discrete-time results \citep{Luo2019AlgorithmicFramework}. Bounding each discrepancy by its absolute value gives the following upper bound on the ground-truth cost.
\begin{corollary}[Estimation Gap Bound] \label{corollary:estimation_gap}
In the setting of \cref{lemma:telescope_lemma}, we have
\begin{align*}
 \mathcal{J}(t_0, X^{t_0,x_0,u}, u)  &\leq \mathcal{J}(t_0, \hat{X}^{t_0,x_0,u}, u) + \sum_{j=0}^{K-1} \mathbb{E}_{x \sim P_{\hat{X}_{t_j}}} [\lvert D_j(x,u) \rvert] %
\end{align*}
\end{corollary}
Thus, a treatment plan with low model-based surrogate cost and small accumulated discrepancies also has low ground-truth cost. The temporal partition is freely chosen, enabling flexible localization of model mismatch and improving empirical performance (see \cref{sec:experimental_resultss}).

\textbf{Path-space IPM.} 
Each one-step discrepancy $D_j(x,u)$ is the difference between the expected value of the cost functional $h_j$ under the true and learned solution path laws. Direct evaluation is infeasible because $h_j$ contains the true continuation value, i.e., the true remaining cost from $t_{j+1}$ onward. To bound $D_j(x,u)$ without evaluating $h_j$, we consider the worst-case difference between the true and learned solution path laws over a class $\mathcal G$ of path functionals for which $h_j \in \mathcal{G}$. The resulting integral probability metric (IPM) quantifies the discrepancy between the true and learned conditional path laws.
\begin{lemma}[Increment Bound] \label{lemma:increment_bound} For any function class $\mathcal{G}\subseteq \mathcal{L}^{0}\bigl((\mathbb{R}^{d_x})^{\intervalj},\mathbb{R}\bigr)$ inside a normed space s.t. $h_j$ can be normalized into $\mathcal{G}$ with normalization constant $c_\mathcal{G} \in \mathbb{R}_{\geq 0}$, we have
    \begin{align*}
    \lvert D_j(x, u) \rvert \leq c_{\mathcal{G}}\cdot \mathrm{IPM}_\mathcal{G}[P_{X_{\intervalj}^{t_j,x,u_{\intervalj}}}, P_{\hat{X}_{\intervalj}^{t_j,x,u_{\intervalj}}}],
    \end{align*}
    where $\mathrm{IPM}_{\mathcal{G}}(P,Q)\coloneqq \sup_{g\in\mathcal{G}}|\E_P[g]-\E_Q[g]|$ is the integral probability metric generated by $\mathcal{G}$ and $c_{\mathcal{G}}$ is a constant depending on $\mathcal{G}$.
    For path functionals $h_j \in \mathcal{G}$, the inequality holds for $c_{\mathcal{G}} =1$.
\end{lemma}
IPMs encompass several familiar notions of distributional discrepancy. Depending on the assumptions imposed on $h_j$, different choices of $\mathcal G$ result in different discrepancy metrics (see \cref{app:reminder_ipm}). We choose $\mathcal{G}$ as the unit disk of the signature kernel Reproducing Kernel Hilbert Space (RKHS), as discussed in detail in \cref{sec:sigCTO}. There, we also explain the choice of $c_{\mathcal G}$ and how to accommodate cases where $h_j \notin \mathcal{G}$.
Similar bounds have been developed for discrete time \citep{yu2020mopo}.

\textbf{Conservative Objective.} Combining \cref{lemma:increment_bound} and \cref{corollary:estimation_gap} bounds the ground-truth cost of a treatment plan by the model-based cost plus a penalty for the accumulated discrepancy between the true and learned solution path laws. Based on this, we propose a \emph{conservative}, ``model-error-aware''
optimization objective by augmenting the model-based cost with
an IPM regularizer %
\begin{equation}\label{eq:conservative_objective}
\mathbb{E}\!\left[
    \int_{t_0}^{t_f} f\bigl(t,\hat{X}^{t_0,x_0,u}_t, u_t\bigr)\, \mathrm{d}t 
    + g\bigl(\hat{X}^{t_0,x_0,u}_{t_f}\bigr)
\right] + \lambda \underbrace{\sum_{j=0}^{K-1} \mathbb{E}_{x \sim P_{\hat{X}_{t_j}}} \left[\mathrm{IPM}_{\mathcal{G}}[P_{X_{\intervalj}^{t_j,x,u_{\intervalj}}}, P_{\hat{X}_{\intervalj}^{t_j,x,u_{\intervalj}}}]\right]}_{=: R(u)}. %
\end{equation}
Here, $\lambda>0$ controls the degree of conservatism. For $\lambda \ge c_{\mathcal G}$, the objective upper bounds the true cost, and its minimizer satisfies the following optimality gap.
\begin{lemma}[Optimality Gap of the Conservative Optimizer]\label{lemma:optimality_gap} Let $
\hat{u}$ be the minimizer of the conservative objective in \cref{eq:conservative_objective} and $u^*$ be the minimizer of the true objective in \cref{eq:optimization_rephrased_as_OC}. If there exists a $c_\mathcal{G} \geq 0$ s.t. the increment bound in \cref{lemma:increment_bound} holds  for each $j \in \lbrace 0 , \ldots, K-1 \rbrace$ and choose $\lambda\ge c_\mathcal{G}$, then the following optimality gap holds:
\begin{align*}
\mathcal{J}(t_0, X^{t_0,x_0,\hat{u}},\hat{u})-\mathcal{J}(t_0, X^{t_0,x_0,u^*},u^*) \le (c_\mathcal{G}+\lambda)R(u^*).
\end{align*}
\end{lemma}
A proof is provided in \cref{app:proof_optimality_gap}. The optimizer of the conservative objective is near-optimal whenever the true optimum $u^*$ has a small discrepancy penalty $R(u^*)$. %
In the infinite-data limit of \cref{assum:structural_identifiability}, assuming the learned solution path law converges to the true one, the IPM regularizer vanishes and \cref{eq:conservative_objective} reduces to the unregularized objective.

\subsection{From IPM to Signature Kernel MMD and \ours} \label{sec:sigCTO}
While the model-based cost can be evaluated using the learned dynamics, the regularizer $R(u)$ still depends on the choice of function class $\mathcal G$ and on the unknown true conditional path laws. In this section, we link the general IPM to the signature-kernel MMD and develop a tractable estimate of the conservative objective.

\textbf{From IPM to MMD.} When $\mathcal G$ is the unit ball of an RKHS, the IPM is called \emph{maximum mean discrepancy} (MMD). It compares distributions through their kernel mean embeddings and can be estimated using kernel evaluations (see \Cref{app:reminder_ipm}). To initialize the IPMs in $R(u)$ with MMDs, we need a kernel for trajectories whose RKHS is rich enough to represent or approximate the cost functionals $h_j$. This motivates our choice of the signature kernel, which we introduce next.

\textbf{Signature Kernel.}
The signature kernel is induced by the signature transform, which represents a continuous path $x$ of bounded variation as a feature vector $\mathcal S(x)$ in a Hilbert space $T(\mathbb{R}^{d})$ \cite{cass2024lecture}. 
These features are iterated integrals that capture the order and interaction of path increments.
The representation is expressive: on suitable compact sets of paths, finite linear combinations of signature coordinates can uniformly approximate any continuous path functional.
The inner product of these representations defines the \defined{signature kernel}, $k_{\mathrm{sig}}(x,y) :=\langle\mathcal S(x),\mathcal S(y)\rangle_{T(\mathbb{R}^{d})}$.
Choosing $\mathcal G$ as the unit ball of the RKHS $H_\mathcal{S}(\mathbb{R}^{d})$ induced by $k_{\mathrm{sig}}$ gives the \emph{signature-kernel MMD}
\begin{align*}
    \mathrm{MMD}_{k_{\mathrm{sig}}}(P_X,P_Y)=\lVert\mu_{P_X}-\mu_{P_Y}\rVert_{H_\mathcal{S}(\mathbb{R}^{d})} = \mathrm{IPM}_{D_1}[P_X,P_Y],
\end{align*}
where $D_1 = \{f \in H_\mathcal{S}(\mathbb{R}^{d}) \mid \|f\|_{H_\mathcal{S}(\mathbb{R}^{d})} \leq 1\}$. For a more formal background on signatures, see \cref{app:signature_background}. We use this MMD for the IPM terms in $R(u)$. When $h_j$ belongs to the signature-kernel RKHS, its RKHS norm provides the scaling factor $c_{\mathcal{G}}$ in the bound on $|D_j(x,u)|$ and the upper bound holds exactly. For $h_j\notin H_{\mathcal S,j}(\mathbb R^{d_x})$, signature universality yields a bound on the cost discrepancy in terms of the signature-kernel MMD, with an arbitrarily small additive remainder and an approximation-dependent multiplier $c_\mathcal{G}$, under suitable assumptions (see \cref{app:signature_kernel_mmd}).

\textbf{Signature kernel MMD Estimator and \ours.} Since the true solution path law for any treatment plan $u$ is generally unknown, we estimate the signature kernel MMDs from observed and model-generated trajectories. Specifically, we use the maximum conditional mean discrepancy (MCMD) estimator \citep{park2020measure}. It compares conditional mean embeddings in RKHSs, avoids explicit density estimation, and does not require repeated ground-truth trajectories for candidate plans \citep{park2021conditional} (see \cref{app:reminder_ipm}). To estimate the discrepancy between the solution path laws using MCMD, we require the solution path law to coincide with the corresponding conditional law given the initial state and treatment segment, as is the case, for example, when treatment on each subinterval is prescribed based on the initial state. 
The empirical MCMD consistently estimates the population conditional MMD under the assumptions of \citet[Theorem~4.2]{park2021conditional}. We refer to that work for the detailed assumptions and consistency proof, and to \cref{app:signature_kernel_mmd} for a discussion of signature-kernel universality and approximation properties relevant to these conditions. We provide the full construction of the MCMD estimator in \cref{app:mcmd_estimator}. Moreover, in \cref{app:algorithm} we detail the full end-to-end \ours algorithm that performs conservative treatment optimization from data.

\section{Experimental Setup} \label{sec:experimental_setup}
We evaluate \ours on two state-of-the-art pharmacokinetic–pharmacodynamic (PK/PD) simulators and on one real-world clinical dataset\footnote{Our code is available at \url{https://github.com/NoraSchneider/SigCTO}}. The PK/PD simulators provide ground-truth outcomes under arbitrary (unseen) treatment plans, which allows us to benchmark the full \ours optimization pipeline. However, on the clinical dataset, true outcomes under newly optimized treatment plans are unobserved, so we focus on the components of \ours that can be evaluated on offline data: whether the approximated dynamics result in reliable cost predictions and whether our regularizer detects the distribution shifts it is designed to penalize.

\subsection{Data and Treatment Optimization Tasks}
The \emph{cancer} simulator \citep{geng2017prediction} models chemo- and radio-therapy effects on lung cancer and is commonly used in longitudinal treatment-effect estimation literature, e.g. \citet{kacprzyk2024ode}. We sample observations from two clinically motivated regimes \citep{curran2011sequential}: sequential and concurrent therapy. The states are tumor volume and drug concentration, and the treatments are chemo- and radiotherapy ($d_x=d_u=2$). %
We consider two treatment objectives: an \emph{explicit target} objective minimizing final tumor volume, and a \emph{relative target} objective aiming for a $70\%$ reduction from initial volume. Both tasks also penalize treatment intensity. %

The \emph{Covid-19} simulator \citep{qian2021integrating,dai2021prototype} models disease progression under dexamethasone treatment. We consider a single-intervention regime, i.e., dexamethasone is administered at a single time point with randomly sampled dosage. %
The treatment is the plasma dexamethasone concentration, and the state consists of viral load, innate immune response, adaptive immune response, and lung-tissue dexamethasone concentration ($d_x=4$, $d_u=1$). %
For treatment optimization, we define a \emph{target trajectory} based on average disease progression under intensive treatment. %

\emph{MIMIC-IV} \citep{mimic4} is a publicly available intensive care unit (ICU) database. We follow the preprocessing protocol of \citet{sivaraman2023ignore}, which applies Sepsis-3 inclusion criteria and aggregates measurements into 4-hour intervals. Restricting the cohort to septic shock patients (vasopressor use, $\geq$ 24h trajectory) 
results in $10,\!826$ trajectories with $d_x=14$ state features and $d_u=2$ treatment variables (IV fluid and vasopressor). 
The treatment optimization objective encourages \emph{Sepsis-3 recovery}. \Cref{app:datasets} provides details on all datasets and optimization tasks.

\subsection{Practical Implementation}
\textbf{\ours.} We approximate the structural dynamics model with a neural SDE. %
For treatment optimization, we restrict to a parameterized family of deterministic open-loop controls, where the parameters describe treatment times and dosages. We optimize the parameters with gradient-based optimization: in each iteration, we sample trajectories from the learned neural SDE under the current control parameters and evaluate the objective corresponding to a direct single-shooting, optimal-control approach \citep{hoglund2023neural}. We compute the \ours-regularizer in the conservative objective from a held-out dataset. \ours implementation details are in \cref{app:practical_implementation}.

\textbf{Baselines.} %
First, we benchmark \ours against state-of-the-art, continuous-time treatment-effect estimators: TE-CDE \citep{seedat2022continuous} and INSITE \citep{kacprzyk2024ode} (see \cref{app:baseline_implementation} for details on implementation and hyperparameter tuning). Note that both baselines are pure intervention-prediction and not treatment-optimization methods. Thus, neither provides a structured treatment-plan optimizer, and exhaustive search over continuous treatment plans is infeasible. Hence, we have to deploy a model-agnostic control-library protocol \citep{seedat2022continuous, wang2025variational} for evaluation. Specifically, we sample a shared library of $100$ candidate treatment plans. For every candidate, each method predicts the response and cost and selects the plan with the lowest predicted cost. To ensure a fair, like-for-like comparison, we apply the same protocol to \ours (with $\lambda =0$) and refer to this variant as \ourslibrary. The control-library protocol assesses prediction-based treatment selection within the candidate set. It does not capture the full contribution of \ours: directly optimizing treatment plans while accounting for model error. We therefore separately evaluate \ours in its intended mode, namely end-to-end optimization of our conservative objective.

Second, we benchmark the \ours regularizer against a strong uncertainty-regularized RL baseline. A strong and fitting contender is MOPO \citep{yu2020mopo}. MOPO is designed for discrete-time MDPs and closed-loop policies, so we do not get around adapting its uncertainty-penalization principle to continuous-time, open-loop treatment optimization for a meaningful comparison. We refer to this adaptation as \emph{CT-MOPO} (see \cref{app:baseline_implementation}). It uses an ensemble of neural SDEs and penalizes treatment plans with integrated uncertainty along the response trajectory, where the uncertainty at each time point is the maximum local diffusion covariance across the ensemble members (analogous to MOPO). CT-MOPO and \ours share the same model architecture, unregularized cost objective, control parameterization, and optimization procedure and thus only differ in the regularizer. Thus, benchmarking \ours against CT-MOPO is an isolated comparison of the proposed regularizers.

\textbf{Evaluation Metrics.} 
For control-library evaluation, we report the Spearman correlation between predicted and ground-truth costs across candidate plans and the ground-truth cost of the selected plan. Here, \ourslibrary uses $\lambda=0$, as the regularizer would intentionally alter the cost-based ranking to account for model error.
For full optimization, we evaluate the regularizer by reporting relative ground-truth costs: the optimized plan's ground-truth cost divided by that under naive optimization; values below 1 indicate improvement. As MIMIC-IV provides no true outcomes for newly optimized plans, we evaluate the dynamics model via cost-prediction correlation and absolute error on held-out trajectories, and the regularizer via its correlation with the discrepancy between observed and predicted trajectories. For further details, see  \cref{app:evaluation_metrics}.

\begin{table}
    \caption{\textbf{Benchmarking \ours on PK/PD tasks.} All results are reported as mean$\pm$standard deviation over $15$ task initializations. \ours uses a linear kernel and $K=1$.}\label{tab:all_results}
    \begin{subtable}{\textwidth}
        \centering
        \caption{\textbf{\ours compared to causal baselines.} Causal baselines and \ourslibrary select treatments using the control library approach; \ours directly optimizes treatment plans using our proposed conservative objective.}
        \label{tab:results_control_library}
      \small
      \setlength{\tabcolsep}{3pt}
      \begin{tabular}{@{}lcccccc@{}}
        \toprule
        ~ & \multicolumn{3}{c}{\textbf{Spearman Correlation $(\uparrow)$}} 
          & \multicolumn{3}{c}{\textbf{$\sqrt{\text{Ground-Truth Cost}}$} $(\downarrow)$} \\
        \midrule
        Method & Cancer-Exp & Cancer-Rel & Covid-19 & Cancer-Exp & Cancer-Rel & Covid-19 \\
        \midrule
        TE-CDE
        & 0.29$\pm$0.10 & -0.30$\pm$0.10 & 0.35$\pm$0.51 & 16.78$\pm$14.11 & 13.18$\pm$10.71 & 4.07$\pm$0.70 \\
        
        INSITE%
        & 0.23$\pm$0.00 & \textbf{0.23}$\pm$0.00 & 0.01$\pm$0.65 & 16.09$\pm$13.13 & 11.94$\pm$9.74 & 4.75$\pm$1.14 \\
    
        \textbf{\ourslibrary} {$(\lambda = 0)$}
        & \textbf{0.39}$\pm$0.00 & 0.22$\pm$0.02 & \textbf{0.76}$\pm$0.16 & 16.80$\pm$13.58 & 12.59$\pm$10.01 & 3.98$\pm$0.78 \\
    
        \textbf{\ours} {$(\lambda = 100)$}
        & - & - & - & \textbf{12.95}$\pm$8.22 & \textbf{8.03}$\pm$6.31 & \textbf{3.18}$\pm$0.94 \\   
        \bottomrule
      \end{tabular}
    \end{subtable}
    
  \vspace{0.5em} %

      \begin{subtable}{\textwidth}
        \centering
        \caption{\textbf{\ours compared to CT-MOPO.} Relative Ground-Truth Cost $(\downarrow)$ for each optimization task.} %
        \label{tab:results_mopo}
      \small
      \setlength{\tabcolsep}{4pt}
      \begin{tabular}{@{}lccc@{\hspace{2.2em}}ccc@{}}
        \toprule
        ~ & \multicolumn{3}{c}{\textbf{$\lambda = 100$}} 
          & \multicolumn{3}{c}{\textbf{$\lambda = 1000$}}  \\
        \midrule
        Method & Cancer-Exp & Cancer-Rel & Covid-19 & Cancer-Exp & Cancer-Rel & Covid-19 \\
        \midrule
        CT-MOPO
            & 0.39$\pm$0.41 
            & 0.66$\pm$0.20 
            & 1.93$\pm$1.65

            &  0.10$\pm$0.06 
            & 0.62$\pm$0.31 
            & 11.90$\pm$33.16 \\
    
        \textbf{\ours}
            & \textbf{0.07}$\pm$0.11 
            & \textbf{0.19}$\pm$0.11 
            & \textbf{0.49}$\pm$0.31
            & \textbf{0.01}$\pm$0.03 
            & \textbf{0.14}$\pm$0.01 
            & \textbf{0.52}$\pm$0.48 \\
        \bottomrule
      \end{tabular}
    \end{subtable}
  \vspace{-0.08cm}
\end{table}

\section{Experimental Results}\label{sec:experimental_resultss}
\textbf{\ours finds better treatment plans than baselines.} \Cref{tab:all_results} compares all methods on the PK/PD treatment optimization tasks. Among methods evaluated with the control-library protocol, \ourslibrary achieves the highest rank correlation on Cancer-Explicit and Covid-19 and the second-highest on Cancer-Relative (\cref{tab:results_control_library}). The fixed controls library may underestimate model exploitation in the full, continuous control space. When optimizing end-to-end and using the intended optimization mode, \ours achieves the best overall ground-truth costs, indicating that effective treatment optimization requires both searching beyond a fixed library and controlling for model errors. \Cref{tab:results_mopo} shows that the \ours signature kernel MMD regularizer achieves lower relative ground-truth costs than uncertainty-based regularization with CT-MOPO. This indicates larger improvements over naive optimization and suggests that our \ours regularizer more effectively mitigates model exploitation than the uncertainty-based baseline.

\begin{figure*}[!t]
  \centering
  \includegraphics[width=\textwidth]{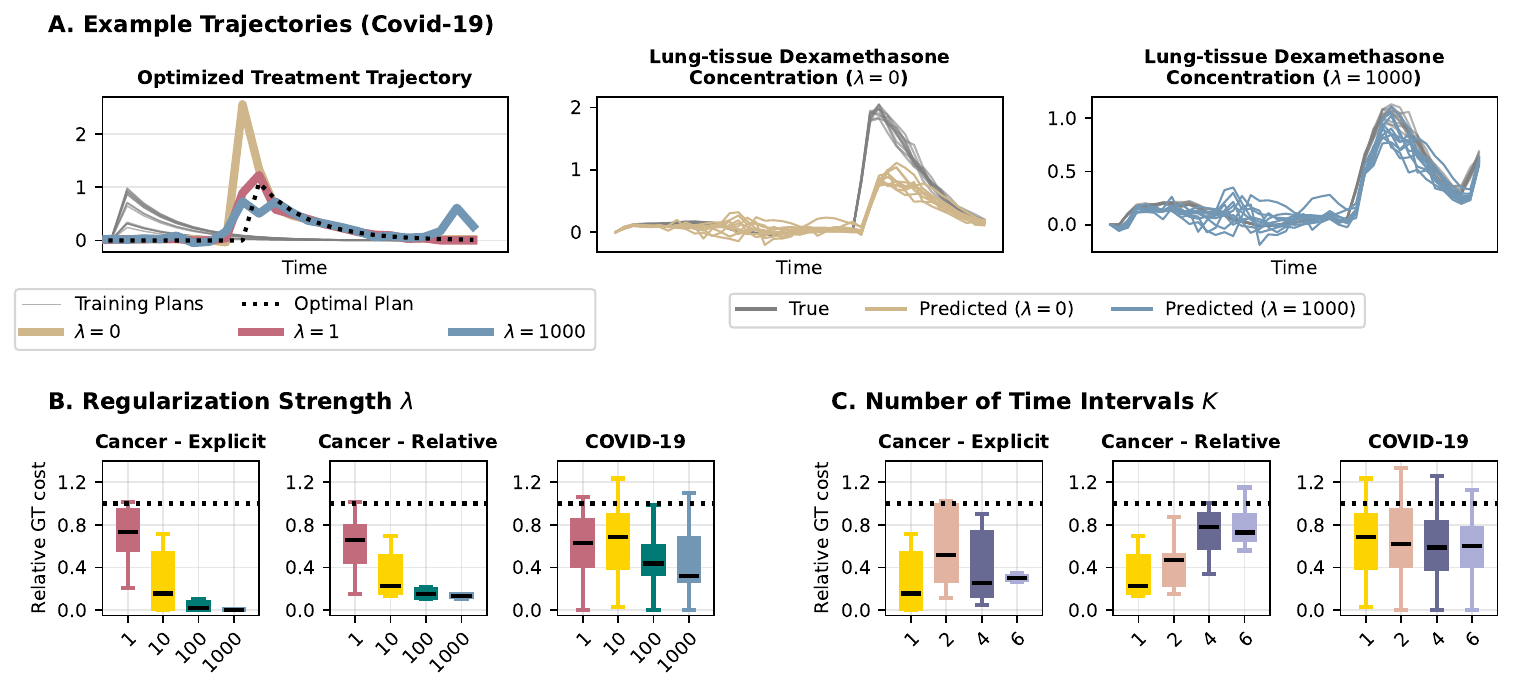}
  \caption{ \textbf{\ours on PK/PD tasks.} \textbf{A.:} Optimized treatment plans (left) and corresponding predicted and true states (right). \textbf{B.} \& \textbf{C.:} Relative ground-truth cost for different $\lambda\in\{1,10,100,1000\}$ with $K=1$, and $K \in \{1,2,4,6\}$ with $\lambda=10$ using a linear kernel. Values below 1 indicate improvement over naive optimization. 
  Boxplots show medians and IQRs, with whiskers extending to the most extreme points within $1.5$ IQRs over 15 optimization task initializations. %
  }
  \label{fig:sigcto_effectiveness}
   \vspace{-0.2cm}
\end{figure*}

\textbf{\ours regularization mitigates model exploitation and improves treatment plans.}
\Cref{fig:sigcto_effectiveness} evaluates the effectiveness of \ours on the three PK/PD-based treatment-optimization tasks. Panel \textbf{A} illustrates an example on Covid-19: under naive, unregularized optimization ($\lambda=0$), the optimized treatment plan substantially deviates from the true optimum (left), and the predicted trajectories under this found plan do not match the true response (middle), indicating model exploitation. Stronger regularization ($\lambda=1000$) pushes the treatment plan closer to the true optimum (left) and improves agreement between predicted and true trajectories (right).

When varying the regularization strength $\lambda$ (Panel \textbf{B}), the relative ground-truth costs are clearly below $1$ on all tasks and decrease as $\lambda$ increases, indicating that our regularizer is effective and improves performance over naive optimization. Increasing the number of time intervals $K$ (Panel \textbf{C}) leaves performance largely unchanged on Cancer-Explicit and Covid-19 but worsens it on Cancer-Relative, although regularization remains beneficial for all $K$ compared to naive optimization. With an RBF lifting kernel (see \cref{app:additional_experimental_results}), relative ground-truth costs rise at $\lambda=1000$ on both cancer tasks, suggesting overregularization. Increasing $K$ improves performance on these tasks, possibly because conditioning on intermediate predicted states favors data-supported trajectories. Sensitivity analyses in \Cref{app:additional_experimental_results} further examine violations of our modeling and identifiability assumptions. \ours remains beneficial on average under correlated process noise, but can worsen performance under partial observability or strong treatment-assignment confounding.

\begin{wraptable}{r}{0.4\textwidth}
    \centering
    \small
    \setlength{\tabcolsep}{2pt}
    \caption{\textbf{MIMIC-IV Results.} Correlations ($\uparrow$) and absolute error ($\downarrow$) of predicted costs on held-out data. All correlations are significant ($p<0.001$) except TE-CDE Pearson ($p=0.84$).} 
    \label{tab:MIMIC_predictive_performance}
    \begin{tabular}{@{}lccc@{}}
        \toprule
        Method & Spearman & Pearson & Abs. Error \\
        \midrule
        TE-CDE          & -0.10 & 0.00 & 26.32$\pm$37.16 \\
        INSITE          & 0.66  & 0.41 & 19.98$\pm$102.24 \\
        \textbf{\ours}  & \textbf{0.68} & \textbf{0.48} & \textbf{10.30}$\pm$38.35 \\
        \bottomrule
    \end{tabular}
    \vspace{-1.2em}  %
\end{wraptable} 
\textbf{\ours regularizer predicts distribution shifts in real-world clinical data.} Truly assessing newly optimized treatment plans on real patients is infeasible. However, with an offline dataset like MIMIC-IV, we can still evaluate two core components of \ours. First, \cref{tab:MIMIC_predictive_performance} shows that \ours most reliably predicts observed costs on held-out patient trajectories compared to causal baselines. Second, we evaluate whether our regularizer predicts model mismatch by comparing the regularizer to the signature score between observed and predicted trajectories for held-out patients. The Spearman correlation of $0.26$ ($p<0.001$) indicates that larger regularizer values correspond to greater trajectory mismatch. Together, these results suggest that key components of \ours may transfer to real-world clinical data. 

\section{Conclusion}\label{sec:conclusion}
We introduce a conservative continuous-time framework for treatment optimization from irregularly sampled observational data. By modeling patient dynamics as a controlled SDE and penalizing out-of-distribution treatment plans via a signature kernel MMD, our approach provably upper bounds the true treatment cost and empirically yields more reliable and better-performing plans than baselines.

Our framework has several limitations. Similar to other pessimistic approaches, restricting optimization to the data support may cause it to miss optimal treatment -- particularly when it lies far outside the support. %
In clinical settings, selection bias and unmeasured confounding may further restrict the empirical support, risking that historic and potentially suboptimal treatment plans get reinforced rather than improved upon. Moreover, our MIMIC-IV evaluation suggests that our key components may transfer to clinical trajectories; however, it is best interpreted as an indirect, initial component-level validation, while clinical utility remains to be established in future work.

Our work further opens several directions for future work. Extending \ours to closed-loop policies or plans adapted at irregular decision points is a natural direction, since the regularizer already operates on a flexible time partition. Another important direction is translating \ours to real-world clinical settings. This includes relaxing modeling assumptions, such as full observability, or incorporating domain knowledge into the dynamics model. Practically, clinical deployment will require rigorous prospective evaluation, appropriate safety constraints, and human oversight.

\begin{ack}

We thank Junhyung Park for his time and for helpful discussions about the paper and the MCMD estimator. 
We thank Cristopher Salvi for his input on Signature Scores for neural SDE training. 
We thank Emilio Ferrucci for his input on causal identifiability, stochastic control, rough path theory, and for sharing useful pointers to the related literature. 
This work has been supported by the German Federal Ministry of Education and Research (Grant: 01IS24082).
This work has been supported by the European Union (ERC, DYNAMICAUS, 101221985).
\end{ack}

\newpage
\bibliographystyle{plainnat}
\bibliography{references}

\newpage

\appendix
\crefalias{section}{appendix}
\crefalias{subsection}{appendix}
\crefalias{subsubsection}{appendix}
\appendix
\section{Summary of Notation} \label{app:notation}
A summary of notation is provided in \cref{tab:notation}.
{%
\footnotesize
\begin{longtable}{@{}p{0.2\textwidth} p{0.75\textwidth}@{}}
\caption{Summary of notation}\label{tab:notation}\\
\toprule
\textbf{Symbol} & \textbf{Meaning}\\
\midrule
\endfirsthead
\caption[]{Summary of notation (continued)}\\
\toprule
\textbf{Symbol} & \textbf{Meaning}\\
\midrule
\endhead
\bottomrule
\endfoot
$I$ & $I:=[t_s,t_f]$, the considered finite time horizon \\
$(\Omega,\mathcal{F},\mathbb{F},P)$ & the complete filtered probability space with filtration $\mathbb{F}=(\mathcal{F}_t)_{t\in I}$\\
$A$ & the treatment space $A\subseteq\mathbb{R}^{d_u}$\\
$N$ & the number of training samples\\
$n_i$ & the number of measurements for patient $i$\\
$\mathcal{D}$ & the dataset $\mathcal{D}:= \{ (u^{(i)}, y^{(i)}, z^{(i)}) \}_{i \in [N]}$\\
$d_z$ & dimension of the stochastic covariate process $Z:I \times \Omega \rightarrow \mathbb{R}^{d_z}$\\ 
$d_y$ & dimension of the stochastic outcome process $Y:I \times \Omega \rightarrow \mathbb{R}^{d_y}$\\
$d_u$ & dimension of the stochastic treatment process $U:I \times \Omega \rightarrow \mathbb{R}^{d_u}$\\
$d_x$ & $d_x = d_z+ d_y$ dimension of $X = (Z,Y)^\top$\\
$U$ & the stochastic treatment process $U:I\times\Omega\to A$\\
$Y$ & the stochastic outcome process $Y:I\times\Omega\to\mathbb{R}^{d_y}$\\
$Z$ & the stochastic covariate process $Z:I\times\Omega\to\mathbb{R}^{d_z}$\\
$X$ & the state process $X=(Z,Y):I\times\Omega\to\mathbb{R}^{d_x}$\\
$u$ & a deterministic treatment plan $u:I\to A$, or a sample of the stochastic process $U$\\
$X_{[a,b]}$ & the restriction of the stochastic process $X$ to $[a,b]\subseteq I$\\
$P_{X_{[a,b]}}$ & the law of the restricted process $X_{[a,b]}$\\
$\mathcal{F}^X_t$ & the sigma-algebra generated by the values of $X$ up to time $t$\\
$P_{Y\mid X_{[a,b]}}$ & the conditional law of a process $Y$ given $X_{[a,b]}$\\
$\mathcal{M}_1 (\Omega, \mathcal{A})$ & the set of probability measures on measurable space $(\Omega, \mathcal{A})$\\
$\mathcal{Y}^{\mathcal{X}}$ & the set of all functions $\mathrm{Funct}(\mathcal{X}, \mathcal{Y}) = \lbrace f : \mathcal{X} \to \mathcal{Y} \rbrace$\\
$\mathcal{L}^{0}((\mathcal{X},\mathcal{A}_\mathcal{X}), (\mathcal{Y},\mathcal{A}_\mathcal{Y}))$ & the $\mathcal{A}_\mathcal{X}-\mathcal{A}_\mathcal{Y}$ measurable functions $f:\mathcal{X} \rightarrow \mathcal{Y}$\\
$W$ & the standard Wiener process in the controlled SDE\\
$\mu,\sigma$ & the shared structural drift and diffusion functions of the controlled SDE\\
$\mu_0$ & the initial distribution of the state process $X_0$\\
$X^{(u,\mathrm{pot})}$ & the potential state process that would be observed under treatment plan $u$\\
$X^{t_0,x_0,u}$ & the true state process on $[t_0,t_f]$ initialized at $X_{t_0}=x_0$ and controlled by $u$\\
$\hat{X}^{t_0,x_0,u}$ & the estimated state process induced by the learned dynamics model under initial condition $x_0$ and treatment plan $u$\\
$\mathcal{U}_{\mathrm{adm}}$ & the class of admissible treatment plans\\
$f,g$ & the running cost $f:I\times\mathbb{R}^{d_x}\times A\to\mathbb{R}$ and terminal cost $g:\mathbb{R}^{d_x}\to\mathbb{R}$\\
$\mathcal{J}(t_0,X,u)$ & the expected cumulative cost of a state process $X$ under treatment plan $u$ from decision time $t_0$\\
$t_0,x_0$ & the decision time and the corresponding decision-time state\\
$K$ & the number of sub-intervals in the temporal partition of $[t_0,t_f]$\\
$\intervalj$ & $\intervalj:=[t_j,t_{j+1}]$, the $j$-th sub-interval of the temporal partition\\
$P_{X_{\intervalj}^{t_j,x,u_{\intervalj}}}$ & the solution path law of the true dynamics on $\intervalj$ starting at $X_{t_j}=x$ and $U_{\intervalj}=u_{\intervalj}$\\
$P_{\hat{X}_{\intervalj}^{t_j,x,u_{\intervalj}}}$ & the solution path law of the learned dynamics on $\intervalj$ starting at$\hat{X}_{t_j}=x$ and $U_{\intervalj}=u_{\intervalj}$\\
$D_j(x,u)$ & the one-step discrepancy between true and learned dynamics on $\intervalj$ given state $x$ and plan $u$\\
$h_j$ & the deterministic path functional that maps a path segment on $\intervalj$ to its sub-interval cost plus continuation value\\
$\mathcal{G}$ & the function class that defines the integral probability metric\\
$\mathrm{IPM}_{\mathcal{G}}$ & the integral probability metric generated by function class $\mathcal{G}$\\
$c_{\mathcal{G}}$ & the constant relating the one-step discrepancy to the chosen IPM\\
$R(u)$ & the accumulated IPM regularizer for treatment plan $u$\\
$\lambda$ & the regularization parameter controlling the degree of conservatism\\
$\mathcal{S}$ & the signature transform of a path\\ 
$T(\mathbb{R}^{d})$ & the completed tensor algebra over $\mathbb{R}^{d}$\\
$k_{\mathrm{sig}}$ & the signature kernel $k_{\mathrm{sig}}(x,y)=\langle \mathcal{S}(x),\mathcal{S}(y)\rangle_{T(\mathbb{R}^{d})}$\\
$H_\mathcal{S}(\mathbb{R}^{d})$ & the reproducing kernel Hilbert space induced by the signature kernel\\
$\mu_{P_X}$ & the kernel mean embedding of the law $P_X$\\
$\mu_{P_{X\mid U}}$ & the conditional mean embedding of the conditional law $P_{X\mid U}$\\
$\mathrm{MMD}_{k_{\mathrm{sig}}}$ & the maximum mean discrepancy induced by the signature kernel\\
$\mathrm{MCMD}_{k_{\mathrm{sig}}}$ & the maximum conditional mean discrepancy induced by the signature kernel\\
$B_1$ & the unit ball in $H_\mathcal{S}(\mathbb{R}^{d})$\\
\end{longtable}
}

\section{Regularity Assumptions} \label{app:regularity_assum}
Below, we detail regularity assumptions that ensure the controlled SDE admits a unique strong solution and that the SDE-governed optimization problem is well-defined.  
\subsection{Structural Dynamics Model} \label{app:detailed_assum_structural_dynamics}

Let $I = [t_s,t_f]$, let $(\Omega,\mathcal{F},\mathbb F,P)$ be a complete filtered probability space with filtration $\mathbb F=(\mathcal F_t)_{t\in I}$ satisfying the usual conditions, $W=(W_t)_{t\in I}$ a $\mathbb{R}^{d_w}$-valued Wiener process adapted to $\mathbb F$,  $U=(U_t)_{t\in I}$ the progressively $\mathbb F$-measurable treatment-process taking values in $A\subseteq\mathbb{R}^{d_u}$ and 

\begin{align}\label{eg:drift_and_diffusion}
\mu:\mathbb{R}^{d_x}\times A\to\mathbb{R}^{d_x}, \qquad \sigma:\mathbb{R}^{d_x}\times A\to\mathbb{R}^{d_x\times d_w},
\end{align}

the measurable drift- and diffusion functions of the controlled SDE

\begin{align}\label{eq:controlled_sde_in_appendix}
\dif X_t=\mu(X_t,U_t)\,\dif t+\sigma(X_t,U_t)\,\dif W_t,\qquad t\in I. 
\end{align}

Assuming
\begin{enumerate}
    \item[(i)] uniform Lipschitz condition in the state, uniformly over treatments, meaning there exists $K\ge 0$ such that for all $x,y\in\mathbb{R}^{d_x}$ and all $u \in A$,
    \begin{align*}
        \|\mu(x,u)-\mu(y,u)\|+\|\sigma(x,u)-\sigma(y,u)\|\le K\|x-y\|.
    \end{align*}
    \item[(ii)] the set of treatment processes are restricted such that 
    \begin{align*}
        \mathbb{E} \left[ \int_{t_0}^{t_f} \lVert \mu(0,u_t) \rVert^{2} + \lVert \sigma (0,u_t)\rVert^{2} \dif t \right] < \infty
    \end{align*}
\end{enumerate}
the controlled SDE \cref{eq:controlled_sde_in_appendix} admits a unique strong solution for every initial condition $(t_0,x_0)\in I\times\mathbb{R}^{d_x}$, for any progressively $\mathbb{F}$-measurable treatment process $U$, which we denote by $X^{(t_0,x_0,U)}=(X_t^{(t_0,x_0,U)})_{t\in I}$. For more details, see \citet{pham_continuous-time_2009}. 

\subsection{Class of Admissible Treatments}\label{app:admissible_treatments}

Assuming a controlled SDE $(\mu, \sigma, I)$ as in \cref{eq:controlled_sde_in_appendix} and measurable functions $f: I \times \mathbb{R}^{d_x} \times A \rightarrow \mathbb{R}$ and $g: \mathbb{R}^{d_x} \rightarrow \mathbb{R}$ s.t.
\begin{enumerate}
    \item[(i)] $g$ is lower-bounded
    \item[(ii)] $\exists C > 0: |g(x)| \leq C(1 + |x|^2) \quad \forall x \in \mathbb{R}^n$
\end{enumerate}
the \defined{set of admissible controls} for $(t_0, x_0) \in I \times \mathbb{R}^{d_x}$,  
\begin{align*}
\mathcal{U}_{\mathrm{adm}}:=\mathcal{U}_{(t_0,x_0)}^f := \left\lbrace u \mid \mathbb{E}\left[\int_{t_0}^{t_f} |f(s, X_s^{(t_0,x_0,u)}, u_s)|\dif s\right] < \infty \right\rbrace \neq \emptyset
\end{align*}
is non-empty, making the \defined{SDE-governed control problem}
\begin{align*}
    v(t_0,x_0) = \inf_{u\in \mathcal{U}_{\mathrm{adm}}} \underbrace{\mathbb{E}\left[\int_{t_0}^{t_f} f(s, X_s^{(t_0,x_0,u)}, u_s) \dif s + g\left(X_{t_f}^{(t_0,x_0,u)}\right)\right]}_{= \mathcal{J}(t_0, x_0, u)} 
\end{align*}
well-defined, where $\mathcal{J}(t_0, x_0, u)$ is the \defined{cost function} and $v(t_0,x_0)$ the \defined{associated value function}.

\subsection{On Identifiability Assumptions} \label{app:on_identifiability_assum}

\Cref{assum:mechanistic_model,assum:structural_identifiability} are satisfied, for example, in the following fully observed setting:
\[
\begin{aligned}
\dif X_t &= \mu(X_t,U_t)\,\dif t+\sigma(X_t,U_t)\, \dif W_t,\\
\dif U_t &= f(U_t)\,\dif t+g(U_t)\,\dif B_t,
\end{aligned}
\]
where $W$ and $B$ are independent standard Wiener processes. Suppose $(X,U)$ takes values in $\mathbb{R}^{d_x+d_u}$ with a known initial state, all coefficients are bounded and globally Lipschitz, and $\sigma$ and $g$ are diagonal with entries uniformly bounded away from zero in absolute value. Defining interventions by replacing $U_t$ with $u(t)$ while preserving the state equation establishes \cref{assum:mechanistic_model}. Under these conditions, the full joint trajectory law identifies the state drift and diffusion covariance through the generator of $(X,U)$ \citep[Appendix~A.2]{wang2024neural}. These uniquely determine the response laws under admissible treatment plans, establishing \cref{assum:structural_identifiability}.

\section{Proofs} \label{app:proofs}
\subsection{Proof for Structural Identifiability Proposition}\label{app:proof_identifiability}
In this subsection, we prove \cref{prop:structural_identifiability}, which we restate here: 
\begin{proposition*}[Structural Identifiability]
Under \cref{assum:mechanistic_model,assum:structural_identifiability} and the regularity conditions in \cref{app:admissible_treatments}, the law of $X^{(u,\mathrm{pot})}$ and consequently the expected cost $\mathcal{J}(t_0,X^{(u,\mathrm{pot})},u)$  in \cref{eq:treatment_optimization} are identifiable for every $u\in\mathcal{U}_{\mathrm{adm}}$.
\end{proposition*}
\begin{proof}
Fix $u\in\mathcal U_{\mathrm{adm}}$. According to \cref{assum:mechanistic_model} the law of $X$ is governed by
\begin{align*}
\dif X_t=\mu(X_t,U_t)dt+\sigma(X_t,U_t) \dif W_t , 
\end{align*}
for all $u\in \mathcal{U}_{\mathrm{adm}}$, in the observational support or not. By \cref{assum:structural_identifiability} any learned $(\hat\mu,\hat\sigma)$ that induces the observed law $P_{X,U}$ will yield a controlled response law for every deterministic $u\in\mathcal U_{\mathrm{adm}}$ that agrees with the true underlying dynamics.
Hence, the law of $X^{(u,\mathrm{pot})}$ is uniquely determined by the observational law $P_{X,U}$. The regularity and integrability conditions in \cref{app:detailed_assum_structural_dynamics} ensure existence and pathwise uniqueness of the
controlled SDE as well as the existence of conditional expectations of the measurable functional
\begin{align*}
X^{(u,\mathrm{pot})} \mapsto \int_{t_0}^{t_f} f(t,X_t^{(u,\mathrm{pot})},u_t)\,\dif t +g(X_{t_f}^{(u,\mathrm{pot})}).
\end{align*}
Hence identifiability of $X^{(u,\mathrm{pot})}$ implies identifiability of $\mathcal{J}(t_0,X^{(u,\mathrm{pot})},u)$.
\end{proof}

\subsection{Proof for Estimation Error Telescope Lemma} \label{app:proof_telescope}
In this subsection, we prove \cref{lemma:telescope_lemma}, which we restate here: 
\begin{lemma*}[Estimation Error Telescope]
For an admissible treatment plan $(u: I \to \mathbb{R}^{d_u}) \in \mathcal{U}_{\mathrm{adm}}$, initial state $x_0$, and any temporal partition $t_0 < t_1 < \cdots < t_K = t_f$, 
\begin{align*}
\mathcal{J}(t_0, X^{t_0,x_0,u}, u) - \mathcal{J}(t_0, \hat{X}^{t_0,x_0,u}, u) = \sum_{j=0}^{K-1} \mathbb{E}_{x \sim P_{\hat{X}_{t_j}}} [D_j(x,u)] 
\end{align*}
with one-step discrepancy on $\intervalj$ between true and learned dynamics given by 
\begin{align*}
    D_j(x,u) &= \mathbb{E}_{x_{\intervalj} \sim P_{X_{\intervalj}^{t_j,x,u_{\intervalj}}}} \left[h_j(x_{\intervalj})\right]  -\, \mathbb{E}_{\hat{x}_{\intervalj} \sim P_{\hat{X}_{\intervalj}^{t_j,x,u_{\intervalj}}}} \left[h_j(\hat{x}_{\intervalj})\right]     
\end{align*} 
and $h_j$ is the deterministic path functional
$x_{\intervalj} \mapsto \int_{t_j}^{t_{j+1}}\! f(s, x_{s}, u_s)\dif s + \mathcal{J} (t_{j+1}, x_{t_{j+1}}, u)$ mapping a path segment $x_{\intervalj}$ to a sum of the cumulative cost on $\intervalj$ and the continuation value at $t_{j+1}$.
\end{lemma*}

Without loss of generality, we assume $I=[0,1]$ for notational convenience. Note the following. 
\begin{remark*}
Given an admissible $u \in \mathcal{U}$ we have by the path-wise uniqueness of the flow of the SDE for $X$, due to the Markovian structure of the SDE:
\begin{align*}
X_t^{t_0,x_0,u} = X_t^{(\tau,X_{\tau}^{t_0,x_0,u},u)}, \quad t \geq \tau
\end{align*}
for any stopping time $\tau$ valued in $[t_0,1]$. Hence we obtain:
\begin{align}
\mathcal{J}(t_0,x_0,u) &= \mathbb{E}\left[\int_{t_0}^{\tau} f(s, X_s^{t_0,x_0,u}, u_s) \dif s + \int_{\tau}^{1} f(s, X_s^{t_0,x_0,u}, u_s) \dif s + g(X^{t_0,x_0,u}_1)\right] \\
&\stackrel{(*)}{=} \mathbb{E}\left[\mathbb{E}\left[\int_{t_0}^{\tau} f(s, X_s^{t_0,x_0,u}, u_s) \dif s + \int_{\tau}^{1} f(s, X_s^{t_0,x_0,u}, u_s) \dif s + g(X^{t_0,x_0,u}_1) \mid \mathcal{F}_{\tau}\right]\right]  \\
&\stackrel{(**)}{=} \mathbb{E}\left[\int_{t_0}^{\tau} f(s, X_s^{t_0,x_0,u}, u_s) \dif s + \mathbb{E}\left[\int_{\tau}^{1} f(s, X_s^{t_0,x_0,u}, u_s) \dif s + g(X^{t_0,x_0,u}_1) \mid \mathcal{F}_{\tau}\right]\right] 
\end{align}
where in $(*)$ we used the tower property of the conditional expectation and in $(**)$ we used that $\tau$ a stopping time making $\int_{t_0}^{\tau} f(s, X_s^{t_0,x_0,u}, u_s) \dif s$ $\mathcal{F}_{\tau}$ measurable. Moreover, by shifting control and Brownian motion by $-\tau$, in the underlying controlled SDE, we obtain:
\begin{align*}
\mathbb{E}\left[\int_{\tau}^{1} f(s, X_s^{t_0,x_0,u}, u_s) \dif s + g(X^{t_0,x_0,u}_1) \mid \mathcal{F}_{\tau}\right] = \mathcal{J}(\tau, X_{\tau}^{t_0,x_0,u}, u) \quad P-\text{a.s.}
\end{align*}
such that
\begin{align*}
\mathcal{J}(t_0, x_0, u) = \mathbb{E}\left[\int_{t_0}^{\tau} f(s, X_s^{t_0,x_0,u}, u_s) \dif s + \mathcal{J}(\tau, X_{\tau}^{t_0,x_0,u}, u)\right]
\end{align*}
\end{remark*}

\begin{proof} 
For $j \in \{0,\ldots, K \}$, define the stochastic process $\tilde{X}^{(j)}$ that evolves up to time $t_j$ according to the (estimated) SDE governed by $\mu_\theta, \sigma_\theta$ and starting from $t_j$ according to the true SDE with $\mu, \sigma$, meaning $\tilde{X}^{(j)} = (\tilde{X}_s^{(j)})_{s \in I}$ satisfies
\begin{align}
\tilde{X}_{[0,t_j]}^{(j)} &= \text{solution of } \quad x_0 + \int_{0}^{t_j} \mu_\theta(\tilde{X}_s^{(j)}, u_s)\dif s + \int_{0}^{t_j} \sigma_\theta(\tilde{X}_s^{(j)}, u_s)dW_s, \\
\tilde{X}_{[t_j,1]}^{(j)} &= \text{solution to } \quad \tilde{X}_{t_j}^{(j)} + \int_{t_j}^1 \mu(\tilde{X}_s^{(j)}, u_s)\dif s + \int_{t_j}^1 \sigma(\tilde{X}_s^{(j)}, u_s)dW_s.
\end{align}
Note that $\tilde{X}^{(0)} := X:= X^{(0,x_0,u)}$, $\tilde{X}^{(K)} = \hat{X} = \hat{X}^{(0,x_0,u)}$. Using the abbreviation
\begin{align*}
\tilde{\mathcal{J}}^{(j)}:=  \mathbb{E}\left[\int_0^{t_j} f(s, \tilde{X}_s^{(j)}, u_s)\dif s + \mathcal{J}(t_j, \tilde{X}_{t_j}^{(j)}, u)\right]
\end{align*}
and a "telescoping-sum" argument, we can write
\begin{align*}
\mathcal{J}(0, X^{(0,x_0,u)}, u) - \mathcal{J}(0, \hat{X}^{(0,x_0,u)}, u) = \tilde{\mathcal{J}}^{(0)} - \tilde{\mathcal{J}}^{(K)} = \sum_{j=0}^{K-1} \left(\tilde{\mathcal{J}}^{(j)} - \tilde{\mathcal{J}}^{(j+1)}\right).
\end{align*}
Moreover

\begin{align*}
\tilde{\mathcal{J}}^{(j)}
&= \mathbb{E}\Biggl[
    \int_0^{t_j}
    f(s,\tilde{X}_s^{(j)},u_s)\,\dif s
    + \mathbb{E}\Bigl[
        \int_{t_j}^1
        f(s,\tilde{X}_s^{(j)},u_s)\,\dif s
        + g(\tilde{X}_1^{(j)})
        \,\Bigm|\,
        \tilde{\mathcal{F}}^{(j)}_{t_j}
    \Bigr]
\Biggr]
\\[0.5em]
&= \mathbb{E}\Biggl[
    \int_0^{t_j}
    f(s,\tilde{X}_s^{(j)},u_s)\,\dif s
    + \mathbb{E}\Bigl[
        \int_{t_j}^{t_{j+1}}
        f(s,\tilde{X}_s^{(j)},u_s)\,\dif s
\\
&\qquad\qquad
    + \underbrace{
        \mathbb{E}\biggl[
            \int_{t_{j+1}}^1
            f(s,\tilde{X}_s^{(j)},u_s)\,\dif s
            + g(\tilde{X}_1^{(j)})
            \,\biggm|\,
            \tilde{\mathcal{F}}^{(j)}_{t_{j+1}}
        \biggr]
    }_{
        \overset{\mathrm{a.s.}}{=}
        \tilde{\mathcal{J}}\bigl(
            t_{j+1},
            \tilde{X}_{t_{j+1}}^{(j)},
            u
        \bigr)
    }
    \,\Bigm|\,
    \tilde{\mathcal{F}}^{(j)}_{t_j}
    \Bigr]
\Biggr].
\end{align*}

where $\tilde{\mathcal{F}}^{(j)}_{t} = \sigma \left(\tilde{X}^{(j)}_s \mid s \leq t \right)$. Similarly
\begin{align*}
\tilde{\mathcal{J}}^{(j+1)}
&= \mathbb{E}\Biggl[
    \int_0^{t_j}
    f(s,\tilde{X}_s^{(j+1)},u_s)\,\dif s
    + \mathbb{E}\Bigl[
        \int_{t_j}^{t_{j+1}}
        f(s, \tilde{X}_s^{(j+1)},u_s)\,\dif s
\\
&\qquad\qquad
    + \underbrace{
        \mathbb{E}\biggl[
            \int_{t_{j+1}}^1
            f(s, \tilde{X}_s^{(j+1)},u_s)\,\dif s
            + g(\tilde{X}_1^{(j+1)})
            \,\biggm|\,
            \tilde{\mathcal{F}}^{(j+1)}_{t_{j+1}}
        \biggr]
    }_{
        \overset{\mathrm{a.s.}}{=}
        \mathcal{J}\bigl(
            t_{j+1},
            \tilde{X}_{t_{j+1}}^{(j+1)},
            u
        \bigr)
    }
    \,\Bigm|\,
    \tilde{\mathcal{F}}^{(j+1)}_{t_j}
    \Bigr]
\Biggr].
\end{align*}
Then, as $\tilde{X}_s^{(j)} = \tilde{X}_s^{(j+1)} \; \forall s \leq t_j$, we have $\tilde{\mathcal{F}}_{t_j}^{(j)} = \tilde{\mathcal{F}}_{t_j}^{(j+1)}$ from which we can write
\begin{align*}
& \tilde{\mathcal{J}}^{(j)} - \tilde{\mathcal{J}}^{(j+1)} \\ 
=& \mathbb{E}\Biggl[ \mathbb{E}\Biggl[\left(\int_{t_j}^{t_{j+1}} f(s, \tilde{X}_s^{(j)}, u_s)\dif s + \mathcal{J}(t_{j+1}, \tilde{X}_{t_{j+1}}^{(j)}, u)\right) \\
& \quad - \left(\int_{t_j}^{t_{j+1}} f(s, \tilde{X}_s^{(j+1)}, u_s)\dif s + \mathcal{J}(t_{j+1}, \tilde{X}_{t_{j+1}}^{(j+1)}, u)\right) \mid \tilde{\mathcal{F}}^{(j)}_{t_j} \Biggr] \Biggr]  \\
\stackrel{(a)}{=} & \mathbb{E}\left[\int_\Omega \left(h_j(\tilde{X}_{\intervalj}^{(j)}(\omega^\prime)) - h_j(\tilde{X}_{\intervalj}^{(j+1)}(\omega^\prime))\right) dP^{(t_j)}(\omega, d\omega^\prime)\right] \\
\stackrel{(b)}{=} &\mathbb{E}_{y \sim P_{\tilde{X}^{(j)}_{t_j}}}\left[\underbrace{\mathbb{E}_{P_{X_{\intervalj}^{t_j,y,u_{\intervalj}}}} \left[h_j(x_{\intervalj})\right] - \mathbb{E}_{P_{\hat{X}_{\intervalj}^{t_j,y,u_{\intervalj}}}} \left[h_j(\tilde{x}_{\intervalj})\right] }_{D_j(y,u)}\right]
\end{align*}
where in $(a)$ we inserted the deterministic function $h_j$ given by 
\begin{align*}
    (\gamma: \intervalj \mapsto \mathbb{R}^n) \rightarrow \int_{t_j}^{t_{j+1}} f(s, \gamma_{s}, u_s)\dif s + \mathcal{J} (t_{j+1}, X^{(t_{j+1},\gamma_{t_{j+1}}, u)}, u)
\end{align*}
and in $(b)$ we used a path notation for the push-forward measure. Note the dependence of $h_j$ on $u_{[t_j, t_{j+1}]}$, thus the subscript $j$. Altogether we therefore obtain 
\begin{align*}
\mathcal{J} - \hat{\mathcal{J}} = \sum_{j=0}^{K-1} \mathbb{E}_{y \sim P_{\tilde{X}^{(j)}_{t_j}}}\left[ \mathbb{E}_{P_{X_{\intervalj}^{t_j,y,u_{\intervalj}}}} \left[h_j(x_{\intervalj})\right]  - \mathbb{E}_{P_{\hat{X}_{\intervalj}^{t_j,y,u_{\intervalj}}}} \left[h_j(\tilde{x}_{\intervalj]})\right] \right].
\end{align*}
\end{proof}

\subsection{Proof for Optimality Gap Lemma}\label{app:proof_optimality_gap}

In this subsection, we prove \cref{lemma:optimality_gap}, which we restate here: 

\begin{lemma*}[Optimality Gap of the Conservative Optimizer] Let $
\hat{u}$ be the minimizer of the conservative objective in \cref{eq:conservative_objective} and $u^*$ be the minimizer of the true objective in \cref{eq:optimization_rephrased_as_OC}. If there exists a $c_\mathcal{G} \geq 0$ s.t. the increment bound in \cref{lemma:increment_bound} holds  for each $j \in \lbrace 0 , \ldots, K-1 \rbrace$ and choose $\lambda\ge c_\mathcal{G}$, then the following optimality gap holds:
\begin{align*}
\mathcal{J}(t_0, X^{t_0,x_0,\hat{u}},\hat{u})-\mathcal{J}(t_0, X^{t_0,x_0,u^*},u^*) \le (c_\mathcal{G}+\lambda)R(u^*).
\end{align*}
\end{lemma*}

For the following proof, we introduce the shorthand notation
\begin{align*}
J(u) &:= \mathcal{J}(t_0, X^{t_0,x_0,u}, u),\\
\hat{J}(u) &:= \mathcal{J}(t_0, \hat{X}^{t_0,x_0,u}, u),\\
R(u) &:= \sum_{j=0}^{K-1} \mathbb{E}_{x \sim P_{\hat{X}_{t_j}}}
\left[\mathrm{IPM}_{\mathcal{G}}[P_{X_{\intervalj}^{t_j,x,u_{\intervalj}}}, P_{\hat{X}_{\intervalj}^{t_j,x,u_{\intervalj}}}]\right].
\end{align*}

\begin{proof}[Proof of \Cref{lemma:optimality_gap}]
By \Cref{lemma:telescope_lemma},
\begin{align*}
J(u)-\hat{J}(u)
= \sum_{j=0}^{K-1}\mathbb{E}_{x\sim P_{\hat{X}_{t_j}}}[D_j(x,u)].
\end{align*}
Hence, for every $u\in\mathcal{U}_{\mathrm{adm}}$,
\begin{align*}
|J(u)-\hat{J}(u)|
&\le \sum_{j=0}^{K-1}\mathbb{E}_{x\sim P_{\hat{X}_{t_j}}}[|D_j(x,u)|] \\
&\le c_\mathcal{G} \sum_{j=0}^{K-1}\mathbb{E}_{x\sim P_{\hat{X}_{t_j}}}
\left[\mathrm{IPM}_{\mathcal{G}}[P_{X_{\intervalj}^{t_j,x,u_{\intervalj}}}, P_{\hat{X}_{\intervalj}^{t_j,x,u_{\intervalj}}}]\right] \\
&= c_\mathcal{G}\,R(u).
\end{align*}
Now let
\begin{align*}
\hat{u}\in\arg\min_{u\in\mathcal{U}_{\mathrm{adm}}}\ \hat{J}(u)+\lambda R(u),
\qquad
u^*\in\arg\min_{u\in\mathcal{U}_{\mathrm{adm}}}\ J(u),
\end{align*}
and assume $\lambda\ge c_\mathcal{G}$. Then
\begin{align*}
J(\hat{u})
&\le \hat{J}(\hat{u}) + c_\mathcal{G} R(\hat{u}) \\
&\le \hat{J}(\hat{u}) + \lambda R(\hat{u}) \\
&\le \hat{J}(u^*) + \lambda R(u^*) \\
&\le J(u^*) + c_\mathcal{G} R(u^*) + \lambda R(u^*),
\end{align*}
where the third line uses the optimality of $\hat{u}$ for the conservative objective and the last line uses again $|J(u^*)-\hat{J}(u^*)|\le c_\mathcal{G} R(u^*)$. Rearranging yields
\begin{align*}
J(\hat{u})-J(u^*)\le (c_\mathcal{G} +\lambda)R(u^*).
\end{align*}
\end{proof}
\section{Integral probability metrics and MMD}\label{app:reminder_ipm}
To quantify the distance between two probability distributions, a widely used approach is to express it in terms of integral probability metrics (IPM):
\begin{definition*}[\defined{Integral probability metric (IPM)}] Let $(\mathcal{X}, \mathcal{A}_\mathcal{X})$ be a measurable space, $P, Q \in \mathcal{M}^1(\mathcal{X}, \mathcal{A}_\mathcal{X})$ and $F \subseteq \mathcal{L}^0(\mathcal{X}, \mathbb{R}) \cap B(\mathcal{X}, \mathbb{R})$ a class of measurable and bounded test functions.
The \defined{integral probability metric (IPM)} generated by $F$ is
\begin{align*}
\mathrm{IPM}[F, P, Q] := \sup_{f \in F} \big\lvert \mathbb{E}_P[f] - \mathbb{E}_Q[f] \big\rvert
\end{align*}
\end{definition*}

\begin{remark*} For the choice of $F$, we have the following tradeoff:
\begin{enumerate}
    \item[(i)] rich enough s.t.\ $\mathrm{IPM}[F,P,Q] = 0 \Leftrightarrow P=Q$
    \item[(ii)] not too large to estimate $\mathrm{IPM}[F,P,Q]$ quickly
\end{enumerate}
\end{remark*}

\begin{example*} Notable function classes, used in machine learning:
\begin{enumerate}
    \item[(i)] For $F_{TV} = \{f \in \mathcal{L}^0(\mathcal{X}, \mathbb{R}) \mid \sup_{x \in \mathcal{X}} |f(x)| \leq 1\}$, $\mathrm{IPM}[F,P,Q] =: \lVert P-Q\rVert_{TV}$ is called the \defined{total variational distance}
    \item[(ii)] For $F_K = \{\chi_{(-\infty,t]} \mid t \in \mathbb{R}\}$ on $\mathcal{X} = \mathbb{R}$, $\mathrm{IPM}[F,P,Q]$ is called the \defined{Kolmogorov} / \defined{$L^\infty$-distance}, the max-norm of their cumulative distributions.
    \item[(iii)] for metric space $(\mathcal{X}, d)$ ($\mathcal{A}_\mathcal{X} = \mathcal{B}(\mathcal{T}_d)$) and $F_W = \{f \in L^0(\mathcal{X}, \mathbb{R}) \cap B(\mathcal{X}, \mathbb{R}) \mid \|f\|_L := \sup_{x \neq y} \frac{|f(x) - f(y)|}{d(x,y)} \leq 1\}$, $\text{IPM}[F_W, P,Q]$ is called \defined{earth mover} / \defined{Wasserstein} / \defined{$L^1$-distance}
\end{enumerate}
\end{example*}

\begin{definition*}[\defined{Maximum Mean Discrepancy}] Let $(\mathcal{X}, \mathcal{A}_\mathcal{X})$ be a measurable space, $(H \subseteq \mathbb{C}^{\mathcal{X}}, \langle \cdot, \cdot \rangle_{H}, k: \mathcal{X} \times \mathcal{X} \to \mathbb{C})$ a RKHS, $P,Q \in \mathcal{M}^1(\mathcal{X}, \mathcal{A}_\mathcal{X})$ and $\mathcal{H} = \{f \in H \mid \|f\|_{H} \leq 1\}$. The integral probability metric $\mathrm{IPM}[\mathcal{H}, P,Q]$ is called \defined{maximum mean discrepancy (MMD)}, written
\begin{align*}
\mathrm{MMD}_k[P,Q] := \mathrm{IPM}[\mathcal{H}, P,Q] =& \sup_{f \in \mathcal{H}} \left| \int_{\mathcal{X}} f(x) \, dP(x) - \int_{\mathcal{X}} f(x) \, dQ(x) \right|\\
=& \sup_{f \in \mathcal{H}} \left| \langle f, \mu_{P} - \mu_{Q} \rangle_{H} \right| = \|\mu_{P} - \mu_{Q}\|_{H}
\end{align*}
where we used $f(x) = \langle f, k(\cdot, x) \rangle_{H}$.
\end{definition*}

\begin{remark*}[Expression of MMD, definiteness] Note:
\begin{enumerate}
    \item[(i)] We can express the MMD in terms of kernel $k$ as
    \begin{align*}
    \mathrm{MMD}_k^2[P,Q] & = \|\mu_{P} - \mu_{Q}\|_{\mathcal{H}}^2 \\
    & = \mathbb{E}_{X,X'}[k(X,X')] - 2\mathbb{E}_{X,Y}[k(X,Y)] + \mathbb{E}_{Y,Y'}[k(Y,Y')]
    \end{align*}
    where $X,X' \sim P$, $Y,Y' \sim Q$ independent, and we use
    \begin{align*}
    \|\mu_{P}\|_{H}^2 & = \left\langle \int k(\cdot,x') \dif P(x'), \int k(\cdot,x) dP(x) \right\rangle_{H}  \\
    & = \iint \langle k(\cdot,x'), k(\cdot,x) \rangle_{H} \dif P(x') \dif P(x)\\ 
    & = \mathbb{E}_{X,X'}[k(X,X')]
    \end{align*}
    \item[(ii)] If $H$ resp. $k$ is characteristic, the MMD is definite, meaning
    \begin{align*}
    \mathrm{MMD}_k[P,Q] = 0 \Leftrightarrow P = Q
    \end{align*}
\end{enumerate}
\end{remark*}
\section{Signature Kernel MMD} \label{app:signature_kernel_mmd_details}

\subsection{Signature Kernel MMD} \label{app:signature_background}
A suitable tool for learning from streamed data is the signature transform, a feature map $\mathcal{S} : C_p(I, \mathbb{R}^{d}) \rightarrow T(\mathbb{R}^{d})$ that converts a data stream (a continuous path $x \in C_p(\mathbb{R}^{d})$ of bounded variation) into a graded sequence of iterated integrals 
\begin{align*}
    \mathcal{S}(x)\!=\! & \left( 1, \mathcal{S}^{(1)}(x), \ldots, \mathcal{S}^{(k)}(x), \ldots \right)\! \! \in \! \prod_{k=0}^{\infty}\!\left( \mathbb{R}^{d}\right)^{\otimes k},\\
    \mathcal{S}^{(k)}(x)& = \int_{t_0 < t_1 < \ldots < t_k<t_f } \!\!\!\!\!\!\!\!dx_{t_1} \otimes \ldots \otimes d x_{t_k} \in \left( \mathbb{R}^{d}\right)^{\otimes k}
\end{align*}
that form an expressive representation of the path in the Hilbert space 
\begin{align*}
T(\mathbb{R}^{d}):= \left\lbrace a = (a_k)_{k \in \mathbb{N}} \in \prod_{k=0}^{\infty}\!\left( \mathbb{R}^{d}\right)^{\otimes k} \mid \sum_{k=0}^{\infty} \lVert a_k \rVert^{2} < \infty \right\rbrace
\end{align*}
and capture the order and interaction of its increments.
An analogy is often drawn to monomials: just as polynomials are dense in $C(K)$ on compact $K\subset\mathbb{R}^d$ by Stone–Weierstraß, finite linear combinations of signature coordinates are dense in the continuous functionals on suitable compact sets of paths, so the signature gives a universal feature map for path-valued data.
Equipped with the scalar product
\begin{align*}
    \langle a,b\rangle_{T(\mathbb{R}^{d})}:=\sum_{k\geq0}\langle a_k,b_k\rangle_{(\mathbb R^d)^{\otimes k}}
\end{align*}
using the standard Euclidean inner products on $(\mathbb R^d)^{\otimes k}$, the signature feature map induces the \defined{signature kernel}, $k_{\mathrm{sig}}(x,y):=\langle \mathcal{S}(x),\mathcal{S}(y)\rangle_{T(\mathbb{R}^{d})}$,
with the corresponding RKHS $H_{\mathcal S}(\mathbb R^d)$. This allows the embedding of laws on path-space by the \defined{kernel mean map} $\mu_{P_X}:=\mathbb{E}[k_{\mathrm{sig}}(X,\cdot)]$ (resp. the \defined{conditional mean embedding} $\mu_{P_{X\mid U}}:=\mathbb{E}[k_{\mathrm{sig}}(X,\cdot) \mid U]$), establishing the \emph{signature kernel MMD} 
\begin{align*}
    \mathrm{MMD}_{k_{\mathrm{sig}}}(P_X,P_Y)=\lVert\mu_{P_X}-\mu_{P_Y}\rVert_{H_\mathcal{S}(\mathbb{R}^{d})} = \mathrm{IPM}_{D_1}[P_X,P_Y]
\end{align*}
where $D_1:=D_1(H_\mathcal{S}(\mathbb{R}^{d})) := \{f \in H_\mathcal{S}(\mathbb{R}^{d}) \mid \|f\|_{H_\mathcal{S}(\mathbb{R}^{d})} \leq 1\}$.

\subsection{Signature Kernel MMD Properties} \label{app:signature_kernel_mmd}
The IPM bound in \Cref{eq:conservative_objective} is exact whenever the one-step continuation functional $h_j$ belongs to the function class $\mathcal{G}$ that induces the IPM, for example $\mathcal G_{\mathrm{TV}}=\{g:\|g\|_\infty\le 1\}$ for the total variation distance (TV) or $\mathcal G=D_1(H_{\mathcal{S},j}(\mathbb{R}^{d_x}))$ for the signature kernel MMD, where  $D_1(H_{\mathcal{S},j}(\mathbb{R}^{d_x}))$ is the unit-ball in the RKHS $H_{\mathcal{S},j}(\mathbb{R}^{d_x})$ over the interval $[t_j, t_{j+1}]$ induced by the signature-kernel (c.f.\ \cref{sec:sigCTO}).

Moreover, for the unit-ball classes considered here, if the path functional is not in $\mathcal{G}$, i.e. $h_j\notin \mathcal{G}$, but is contained in the larger ambient normed space of $\mathcal{G}$, e.g. inside $L^\infty$ for the total variation distance (resp. inside the signature-kernel RKHS $H_{\mathcal{S},j}(\mathbb{R}^{d_x})$ for the signature kernel MMD), a rescaled version of the vector $h_j$ is contained in $\mathcal{G}$, e.g. $\frac{h_j}{\lVert h_j \rVert_\infty} \in \mathcal G_{\mathrm{TV}}$ (resp. $\frac{h_j}{\lVert h_j \rVert_{H_{\mathcal{S},j}(\mathbb{R}^{d_x})}} \in D_1(H_{\mathcal{S},j}(\mathbb{R}^{d_x}))$). This gives the scaling factor $c_{\mathcal{G}}$ in \cref{lemma:increment_bound}, with $c_{\mathcal{G}} := \|h_j\|_{\infty}$ for the TV (resp. $c_{\mathcal{G}} := \|h_j\|_{H_{\mathcal{S},j}(\mathbb{R}^{d_x})}$ for the signature kernel MMD).

Moreover, for the signature kernel MMD, if the path functional is not even in the larger ambient normed space of $\mathcal{G}$, i.e.,  $h_j\notin H_{\mathcal{S},j}(\mathbb{R}^{d_x})$, one can provide an approximation-based justification for the discrepancy penalty using weighted signature approximation.
Under suitable regularity and integrability conditions, $h_j$ belongs to the weighted function space considered by \citet{cuchiero2026global} with the corresponding weight integrable under both solution path laws; see \citet{cuchiero2026global} for weighted signature approximation results. By the universality of the signature kernel on such weighted function spaces (see \citet{cuchiero2026global}), for every $\varepsilon>0$ there exists $g_{j,\varepsilon}\in H_{\mathcal{S},j}(\mathbb{R}^{d_x})$ such that
\begin{align*}
\mathbb{E}_{P_j^y}\!\left[|h_j-g_{j,\varepsilon}|\right]
\le \varepsilon,
\qquad
\mathbb{E}_{\hat P_j^y}\!\left[|h_j-g_{j,\varepsilon}|\right]
\le \varepsilon.
\end{align*}
We refer to \citet{cuchiero2026global} for details. Therefore 
\begin{align*}
|D_j(y,u)|
\le \bigl|\mathbb{E}_{P_j^y}[g_{j,\varepsilon}] - \mathbb{E}_{\hat P_j^y}[g_{j,\varepsilon}]\bigr| + 2\varepsilon
\le \underbrace{\|g_{j,\varepsilon}\|_{H_{\mathcal{S},j}(\mathbb{R}^{d_x})}}_{c_{\mathcal{G},\varepsilon}}\,\mathrm{MMD}_{k_j}[P_j^y,\hat P_j^y] + 2\varepsilon.
\end{align*}

Thus, under RKHS realizability, the scaled signature-kernel MMD bounds the cost discrepancy without an approximation remainder. Under the approximation assumptions above, it instead provides a bound with an arbitrarily small additive remainder and an approximation-dependent multiplier. This motivates its use as a computable discrepancy regularizer beyond RKHS realizability.

Under the assumptions of \citet{park2021conditional}, including
the standard boundedness and universality conditions on the
kernels and suitable estimator regularization, the empirical
MCMD is a consistent estimator of the corresponding population
conditional MMD used in our regularizer.

\subsection{Signature Kernel MMD Estimator} \label{app:mcmd_estimator} 
To estimate the conditional discrepancy between observed and
model-predicted trajectories, we use samples $\{x^{(i)},u^{(i)},x_0^{(i)}\}_{i \in [N]}$ and $\{x'^{,(i)}, u'^{,(i)}, x_0'^{,(i)}\}_{i \in [M]}$ from joint distributions $P_{X,U,X_0}$ and $P_{X',U',X_0'}$, respectively. For a fixed, but arbitrary query conditioning $(x_0,u)\in\mathbb R^{d_x}\times\mathcal U_{\mathrm{adm}}$, we have the plug-in estimate for the square of the MCMD function, 
\small{
\begin{align*}
    \hat{M}^2(P_{X\mid X_0=x_0,U=u},P_{X'\mid X_0'=x_0,U'=u})(x_0,u) & = k_{X_0,U}^\top (x_0,u) W_{X_0,U} K_X W_{X_0,U}^\top k_{X_0,U}(x_0,u) \\
    & \quad - 2k_{X_0,U}^\top(x_0,u)W_{X_0,U}K_{X,X'}W_{X_0',U'}^\top k_{X_0',U'}(x_0,u) \\
    & \quad + k_{X_0',U'}^\top (x_0,u) W_{X_0',U'} K_{X'} W_{X_0',U'}^\top k_{X_0',U'}(x_0,u).
\end{align*}}
where $(K_X)_{ij} = k_{sig}(x^{(i)},x^{(j)})$ (analogous for $(K_{X'})_{ij}$) and $(K_{X,X'})_{ij} = k_{sig}(x^{(i)},x'^{,(j)})$, $(K_{X_0,U})_{ij} = k_{X_0,U}((x_0^{(i)},u^{(i)}),(x_0^{(j)},u^{(j)}))$ (analogous for $(K_{X_0',U'})_{ij}$), $W_{X_0,U} = (K_{X_0,U} + N\varphi \mathbf{I}_N)^{-1}$ (analogous for $W_{X_0',U'}$), $(k_{X_0,U}(x_0,u))_{i} = k_{X_0,U}((x_0^{(i)},u^{(i)}),(x_0,u))$ (analogous for $(k_{X_0',U'}(x_0',u'))_{i}$. $\varphi$ is the ridge regression regularization parameter (see \citet{park2020measure} for details). Although the target distributions are conditional on the same query $(x_0,u)$, the estimator uses samples from the joint laws, whose initial conditions and treatment plans can vary across observations.

For estimation, we equip the conditioning space with the kernel
\[
k_{X_0,U}\bigl((x_0,u),(\widetilde x_0,\widetilde u)\bigr)
=
k_{X_0}(x_0,\widetilde x_0)\,k_U(u,\widetilde u),
\]
where $k_{X_0}$ is a Gaussian RBF kernel and $k_U$ is a signature
kernel on treatment paths. The corresponding RKHS is
$\mathcal H_0\otimes\mathcal H_U$.

In \ours, we set $N=M=N_{\mathrm{val}}$ and use
the held-out dataset $\mathcal D_{\mathrm{val}} =\{(x^{(i)},u^{(i)},x_0^{(i)})\}_{i=1}^{N_{\mathrm{val}}}$ of observed trajectories and their associated treatment plans and initial conditions. For each $i\in[N_{\mathrm{val}}]$, we set $(x_0'^{(i)},u'^{(i)})=(x_0^{(i)},u^{(i)})$ and generate a predicted trajectory $x'^{(i)}$ using the estimated dynamics model. This yields the prediction dataset $\mathcal D_{\mathrm{pred}} =\{(x'^{(i)},u'^{(i)},x_0'^{(i)})\}_{i=1}^{N_{\mathrm{val}}}$. We use these two datasets in the estimator above to estimate the signature MMD between the observed and model-predicted trajectory distributions conditional on a common query $(x_0,u)$. The computational complexity of evaluating the signature kernel MMD estimator is as follows. The signature kernel evaluations via the Goursat-PDE solver scale as $O(dl^2)$, with path-dimension $d$ and time-length $l$. Antidiagonal GPU parallelization reduces the wall-clock dependence on l to linear, with sufficient parallel resources \citep{salvi_signature_2021}. With $N$ samples, computing $W_{X_0,U}$ is a one-time cubic cost in $N$, and each regularizer evaluation scales quadratically in $N$.

\section{\ours Algorithm} \label{app:algorithm}
In \Cref{alg:conservative-treatment-planning}, we provide the full \ours algorithm. The regularizer $R(u)$ is defined for any partition $t_0 < t_1 < \cdots < t_K = t_f$. In practice, we choose equidistant intervals. This does not require regularly sampled data: snippets may contain different numbers of observations at irregular times. We pool observed and predicted reference snippets separately across intervals and precompute the MCMD reference matrices once. During optimization, we reuse these matrices to evaluate the conditional discrepancy at each predicted boundary state and candidate treatment segment, then sum the discrepancies across intervals, as described in the algorithm.

\begin{algorithm}[ht]
\caption{Conservative Treatment Optimization with SigCTO}
\label{alg:conservative-treatment-planning}
\begin{algorithmic}[1]

\Require
Training and validation datasets
$\mathcal{D}_{\mathrm{train}}$, $\mathcal{D}_{\mathrm{val}}$;
objective function $\mathcal{J}(t_0,X,u)$;
treatment horizon $[t_0,t_f]$;
number of intervals $K \geq 1$;
regularization strength $\lambda$;
parameterized treatment family
$\{u_\phi : \phi \in \Psi\}$;
patient state $x_0$ at decision time $t_0$.

\Ensure Optimized treatment plan $u_{\phi_{\mathrm{opt}}}$.

\Statex

\State \textbf{**Learn a dynamics model**}
\State Initialize the parameters of the dynamics model $\theta$.
\State Train $\theta$ on $\mathcal{D}_{\mathrm{train}}$.
\State Fix the learned dynamics model.

\Statex

\State \textbf{**Conservative treatment optimization**}
\State Initialize reference-sample collections for computing regularizer,
$\mathcal{D}_{1} \gets \emptyset$ and
$\mathcal{D}_{2} \gets \emptyset$.
\State Construct equally spaced interval boundaries
$
t_j \gets t_0 + \frac{j}{K}(t_f-t_0),
\qquad j=0,\ldots,K.
$

\State Set
$I_j \gets [t_j,t_{j+1}]$ for $j=0,\ldots,K-1$.
\For{each validation trajectory
      $(x^{(i)},u^{(i)},x_0^{(i)}) \in
      \mathcal{D}_{\mathrm{val}}$}
    \State Predict $\hat{x}^{(i)}$ using the learned dynamics model with initial state $x_0^{(i)}$ and treatment plan $u^{(i)}$.
    \For{$j=0,\ldots,K-1$}
        \State Add the observed reference sample
        $
        \bigl(
            x_{t_j}^{(i)},
            u_{I_j}^{(i)},
            x_{I_j}^{(i)}
        \bigr)
        $
        to $\mathcal{D}_{1}$.

        \State Add the predicted reference sample
        $
        \bigl(
            \hat{x}_{t_j}^{(i)},
            u_{I_j}^{(i)},
            \hat{x}_{I_j}^{(i)}
        \bigr)
        $
        to $\mathcal{D}_{2}$.
    \EndFor
\EndFor
\Statex
\State Using $\mathcal{D}_1$ and $\mathcal{D}_2$,
compute and store $K_X$, $K_{X'}$, $K_{X,X'}$,
$K_{X_0,U}$, and $K_{X'_0,U'}$,
(as in \cref{app:mcmd_estimator}).

\State Precompute and store $W_{X_0,U}$ and $W_{X'_0,U'}$.
\Statex

\State Initialize the treatment-plan parameters $\phi$.

\Repeat
    \State Construct the candidate treatment plan $u_\phi$.

    \State Predict patient trajectories under $u_\phi$ using the learned dynamics model: $
    \hat{X}^{t_0,x_0,u_\phi}.
    $

    \State Estimate the model-based surrogate cost
    $
    \hat{\mathcal{J}}(u_\phi)
    =
    \mathcal{J}\!\left(
        t_0,
        \hat{X}^{t_0,x_0,u_\phi},
        u_\phi
    \right).
    $
    
    \State Extract the predicted boundary states $\hat{x}_{\phi,t_0},\ldots,\hat{x}_{\phi,t_{K-1}}$.
    \State $R(u_{\phi}) \gets 0$
    \For{$j=0,\ldots,K-1$}
        \State Evaluate
        ${M}^{2}(\hat{x}_{\phi,t_j},u_{\phi,I_j})$
        using the precomputed reference matrices.
        \Comment{See \cref{app:mcmd_estimator}}
    
        \State Update
        $
        R(u_\phi)
        \gets R(u_\phi)
        +
        \sqrt{
                {M}^{2}
                (\hat{x}_{\phi,t_j},u_{\phi,I_j})}
        .
        $
    \EndFor

    \State Update $\phi$ to minimize
    \[
    \hat{\mathcal{J}}(u_\phi)
    +
    \lambda R(u_\phi).
    \]

\Until{convergence}

\State Set $\phi_{\mathrm{opt}} \gets \phi$.
\State \Return $u_{\phi_{\mathrm{opt}}}$.

\end{algorithmic}
\end{algorithm}

\section{Datasets and Treatment Optimization Tasks}\label{app:datasets}
In this section, we provide details on the PK/PD simulators and on the MIMIC-IV dataset. Further, we provide for each dataset the treatment optimization objective.

\subsection{Cancer} 
The cancer dataset is simulated using the lung cancer PK/PD model of \citet{geng2017prediction}, which has been used extensively in longitudinal treatment-effect benchmarks.

\paragraph{Dynamics.}
The state is $X_t=(V_t, C_t)\in \mathbb{R}^2_{+}$, where $V_t$ denotes tumor volume and $C_t$ denotes the chemotherapy drug concentration.
The control is $U_t=(U_{c,t},U_{r,t})\in\mathbb{R}^2_{+}$, where $U_{c,t}$ is the administered chemotherapy dose and $U_{r,t}$ is the radiotherapy dose.

Tumor volume evolves according to an ODE with multiplicative noise,
\begin{align}
\frac{\mathrm{d}V}{\mathrm{d}t}
=
\Bigg(
\rho \log\!\Big(\frac{K}{V(t)}\Big)
- \beta_c\, C(t)
- \big(\alpha_r\, U_r(t) + \beta_r\, U_r(t)^2\big)
+ e_t
\Bigg)\, V(t),
\end{align}
where $e_t \sim \mathcal{N}(0,\sigma^2)$. Equivalently, this can be written as the SDE
\begin{align}
\mathrm{d}V_t
&=
\Big[
\rho \log\!\Big(\frac{K}{V_t}\Big)
- \beta_c\, C_t
- \alpha_r\, U_{r,t}
- \beta_r\, U_{r,t}^2
\Big] V_t\, \mathrm{d}t
+
\sigma\, V_t\, \mathrm{d}W_t,
\end{align}
with standard Brownian motion $W_t$.
We set $(\rho,K,\beta_c,\alpha_r,\beta_r)$ to the mean values of the prior distribution reported by \citet{geng2017prediction} and fix the noise scale to $\sigma = 0.1$ in our simulations.

The chemotherapy concentration follows exponential decay with input from administered chemotherapy:
\begin{align}
\frac{\mathrm{d}C}{\mathrm{d}t}
=
- k_C\, C(t) + U_c(t),
\end{align}
where $k_C=0.5$.

\paragraph{Treatment Process.}
For each simulated patient trajectory, we sample one of two treatment protocols inspired by \cite{curran2011sequential,hess2023bayesian,vanderschueren2023accounting}: \emph{sequential} therapy (three weeks of weekly chemotherapy followed by three weeks of radiotherapy) or \emph{concurrent} therapy (biweekly chemo--radiotherapy administration over six weeks). Treatments are modeled as bang-bang controls: at each dosing time $t$, chemotherapy is either off or delivered at a fixed dose, $U_{c,t}\in\{0,5\}$ (mg), and radiotherapy is either off or delivered at a fixed fraction dose, $U_{r,t}\in\{0,2\}$ (Gy), where Gy denotes Gray (ionizing radiation dose).

\paragraph{Initial State.}
Following \citet{geng2017prediction}, we first sample a cancer stage and then sample the initial tumor volume $V_0$ from the corresponding stage-conditional prior defined in \citet{geng2017prediction}. The initial chemo concentration $C_0$ is set to zero. 

\paragraph{Simulation and Preprocessing.}
We generate $800$ patient trajectories over a $60$-day time window ($t \in \left[0, 60\right]$) for training and sample $128$ for validation purposes. To induce irregular observation patterns, we randomly mask $30\%$ of observations uniformly at random. We apply a log-transform to the state variables and then standardize (zero mean, unit variance) using statistics computed on the training set only. Controls are min--max normalized to $[0,1]$ using the known dose bounds.

\paragraph{Treatment Optimization Objective 1 (Explicit Target State).}
Let $X_t=(V_t,C_t)$ denote the cancer state, where $V_t$ is tumor volume. Given an initial patient state $V_{t_0}=v_0$ and initial concentration $C_{t_0}=0$, we seek a treatment plan that minimizes terminal tumor volume while regularizing treatment intensity:
\begin{align}
\min_{u \in \mathcal{U}} \;
\mathbb{E}\!\left[
\left.
\int_{t_0}^{t_f}
\lambda_c\, u_{c,t}^2 + \lambda_r\, u_{r,t}^2 \,\mathrm{d}t
\;+\;
V_{t_f}^2
\ \right|\ X_{t_0}=(v_0,0)
\right],
\qquad t_f=60,
\label{eq:obj_cancer_abs}
\end{align}
where $\lambda_c=\lambda_r=10^{-3}$ in our experiments.

\paragraph{Treatment Optimization Objective 2 (Relative Target State).}
In the second cancer objective, given an initial patient state $V_{t_0}=v_0$ and $C_{t_0}=0$, the goal is to achieve a prescribed \emph{relative} reduction in tumor volume at the terminal time. We define the relative target $v^\star := 0.3\, v_0$, 
corresponding to a $70\%$ reduction from the initial state. We minimize the deviation from this target while regularizing treatment intensity:
\begin{align}
\min_{u \in \mathcal{U}} \;
\mathbb{E}\!\left[
\left.
\int_{t_0}^{t_f}
\lambda_c\, u_{c,t}^2 + \lambda_r\, u_{r,t}^2 \,\mathrm{d}t
\;+\;
\big(V_{t_f} - v^\star\big)^2
\ \right|\ X_{t_0}=(v_0,0)
\right],
\qquad t_f=60,
\label{eq:obj_cancer_rel}
\end{align}
with $\lambda_c=\lambda_r=10^{-3}$.

For both tasks, we solve the optimization problem $15$ times, corresponding to $15$ different initial conditions. In both tasks we sample from the same initial state prior distribution as the training dataset and restrict ourselves to initial tumors with a diameter between $2$ and $5$cm.

\subsection{Covid-19}

We use the COVID-19 disease progression simulator under dexamethasone treatment proposed by \citet{qian2021integrating} (adoption of \citet{dai2021prototype}). The simulator describes the interaction between disease progression and immune response, coupled with a treatment/exposure component.

\paragraph{Dynamics.}
The original simulator in \citet{qian2021integrating} is specified as a deterministic ODE system. To model stochasticity and unobserved heterogeneity, we add multiplicative diffusion terms to the dynamics, yielding the following SDE formulation with state $X_t=(X_{1,t},X_{2,t},X_{3,t},X_{4,t})$.
The state variables $X_{1,t}, X_{2,t}, X_{3,t}, X_{4,t}$ represent viral load, innate immune response to viral infection, adaptive immunity, and (effective) dexamethasone concentration in lung tissue, respectively:
\begin{align*}
dX_{1,t} &= \left( k_{dp} X_{1,t} - k_{di} X_{1,t} X_{3,t}^{h_c} - k_{dr} X_{1,t} X_{2,t} \right) dt + \sigma_1 X_{1,t}dW^1_t, \\
dX_{2,t} &= \left( k_{id} X_{1,t} - k_{io} X_{2,t} + k_{if} X_{1,t} X_{2,t}
+ \frac{k_{ep} X_{2,t}^{h_p}}{k_{cp}^{h_p} + X_{2,t}^{h_p}} - k_d X_{4,t} X_{2,t} \right) dt + \sigma_2  X_{2,t} dW^2_t, \\
dX_{3,t} &= \left( k_{im} X_{2,t} \right) dt + \sigma_3 X_{3,t} dW^3_t, \\
dX_{4,t} &= \left( k_{kel} U_{c,t} - k_{kel} X_{4,t} \right) dt + \sigma_4 X_{4,t} dW^4_t.
\end{align*}
Here, $W^i_t$ denote independent standard Wiener processes and we set $\sigma= \sigma_i = 0.1$ for $i \in \{1,2,3,4\}$. The dynamics parameters are positive real numbers and described in \citet{qian2021integrating,dai2021prototype}. Following \citet{qian2021integrating} we set the coefficients $h_p = h_c = 2$ and the rest to one. 

The signal $U_{c, t}$ denotes our control variable and represents the exogenous dexamethasone administration/exposure input.

\paragraph{Treatment Process.}
We consider a single-intervention setting over a 14-day horizon. Dexamethasone is administered after 7 days at different dosages. Concretely, for each simulated trajectory, we sample a dosage $d \sim \mathrm{Unif}(0.5,5)$ mg and define the control signal as an exponential decay
\begin{align*}
    U_{c}(t) = d \cdot \mathbb{I}(t \ge t^\star)\exp\big(-k_{kel}(t - t^\star)\big).
\end{align*}

\paragraph{Initial State.}
We sample initial conditions $X_0$ from an exponential distribution with rate $\lambda = 100$, as described in \citet{qian2021integrating}. 

\paragraph{Simulation and Preprocessing.}
In total we sample $100$ trajectories over a $14$-day horizon. State trajectories are min--max normalized to $[0,1]$ using statistics computed on the training set.
Treatment inputs are min--max normalized to $[0,1]$ using the known dose bounds $d\in[0,10]$. Validation trajectories consist of $128$ trajectories sampled in the same way.

\paragraph{Treatment Optimization Objective (Target Trajectory).}
Let $X_t \in \mathbb{R}^{d_x}$ denote the COVID-19 simulator state and let $u_t\in\mathbb{R}^{d_u}$ denote the (one-dimensional) dexamethasone exposure control. Given a randomly sampled initial condition $X_{t_0}=x_0$, we define a target trajectory $X_t^\star$ over $[t_0,t_f]$ by sampling $M = 100$ trajectories $\{X^{(m)}\}_{m=1}^M$ under a sampled control curve. Specifically, we sample a random timepoint in $[t_0, t_f]$ and a random dosage in $[5,10]$ and define the control using a similar exponential decay relationship. The target trajectory is then obtained by averaging pointwise in time $X_t^\star := \frac{1}{M}\sum_{m=1}^M X^{(m)}_t, \qquad t\in[t_0,t_f]$. We then optimize treatment to track $X^\star$:
\begin{align}
\min_{u \in \mathcal{U}} \;
\mathbb{E}\!\left[
\left.
\int_{t_0}^{t_f}
\lambda_x\, \|X_t - X_t^\star\|_2^2
\,\mathrm{d}t
\ \right|\ X_{t_0}=x_0
\right].
\label{eq:obj_covid_track}
\end{align}
We choose a factor of $\lambda_x = 1$ so that the regularization strength of our regularizer can live on the same scale as for the Cancer task. We sample $15$ initial conditions from the same prior distribution as the training dataset and then $15$ treatment trajectories resulting in $15$ optimization tasks. 

\subsection{MIMIC-IV}

MIMIC-IV \citep{mimic4} is a publicly available critical care database containing EHR data from ICU admissions at Beth Israel Deaconess Medical Center. We build on the preprocessing pipeline of \citet{sivaraman2023ignore}, which focuses on sepsis patients using Sepsis-3 criteria and aggregates clinical variables at 4-hour intervals. Specifically, we run the provided preprocessing pipeline provided in \url{https://github.com/cmudig/AI-Clinician-MIMICIV}. 

\paragraph{Further Preprocessing.} We further restrict to the septic shock subgroup by filtering patients according to two additional criteria: (1) vasopressor administration at any point during the ICU stay and (2) a minimum trajectory length of 24 hours (6 timesteps), resulting in $10,826$ patients. The dataset is then randomly split 80/20 at the patient level into training ($8,660$ patients) and test ($2,166$ patients) sets. Then, we select $14$ relevant features as state variables: hemodynamic indicators (MAP, systolic blood pressure, heart rate), tissue perfusion markers (arterial lactate, cumulative fluid balance, urine output), SIRS criteria (temperature, respiratory rate, white blood cell count), and SOFA organ dysfunction components (creatinine, platelet count, PaO2\_2, GCS), and the previous vasopressor dose is appended to the state. Together these variables capture the key physiological dimensions used to assess severity and guide treatment decisions in septic shock. The two treatment variables (IV fluid volume and vasopressor dose) are log1p-transformed before modeling. All state features are z-standardized using training set statistics only.

\paragraph{Treatment Optimization Objective (Sepsis-3 recovery).}
Let $X_t \in \mathbb{R}^{d_x}$ denote the patient's physiological state and let $u_t \in \mathbb{R}^{d_u}$ denote the treatment control (IV fluid volume and vasopressor dose). Given an initial condition $X_{t_0} = x_0$, we optimize treatment to restore the two Sepsis-3 recovery criteria~\citep{singer2016third} (MAP $\geq 65$ mmHg and lactate $\leq 2$ mmol/L) while minimizing treatment intensity:
\begin{align}
\min_{u \in \mathcal{U}} \;
\mathbb{E}\!\left[
\left.
\int_{t_0}^{t_f}
\lambda_{\text{MAP}}\, \text{ReLU}(\bar{m} - X_t^{\text{MAP}})^2
+ \lambda_{\text{lac}}\, \text{ReLU}(X_t^{\text{lac}} - \bar{l})^2
+ \lambda_u\, \|u_t\|_2^2
\,\mathrm{d}t
\ \right|\ X_{t_0}=x_0
\right],
\label{eq:obj_sepsis}
\end{align}
where $\bar{m} = 65$ mmHg and $\bar{l} = 2$ mmol/L are the Sepsis-3 clinical thresholds (converted to z-scored units), and $\lambda_{\text{MAP}} = \lambda_{\text{lac}} = 1$ and  $\lambda_u = 0.01$ are weighting coefficients.
\section{Implementation Details of \ours} \label{app:practical_implementation}

\subsection*{Neural SDEs as Structural Dynamics Model}
We model the drift and diffusion of the SDE with neural networks $\mu_\theta :\mathbb{R}^{d_x} \times A \to \mathbb{R}^{d_x}$, $\sigma_\theta : \mathbb{R}^{d_x} \times A \to \mathbb{R}^{d_x\times d_w}$ to model the strong solution $\hat{X} : I \rightarrow \mathbb{R}^{d_x}$ of 
\begin{align}
    \dif \hat{X}_t = \mu_\theta(\hat{X}_t,U_t) \dif t + \sigma_\theta (\hat{X}_t,U_t) \dif W_t
\end{align}
$W : I\to\mathbb{R}^{d_w} $ is a $ d_w$-dimensional Brownian motion and $U = (U_t)_{t \in I}$ our control. 

In all experiments, we parameterized the drift and diffusion so that each state has its own independent neural network function, but still takes all state variables as inputs. This allows for independently learned functions for each state. Each drift neural network consists of a three-layer fully-connected neural network with $64$ hidden units, while the diffusion networks consist of a one-layer fully-connected neural network with $8$ hidden units. We use $\text{LipSwish}$-activation function for intermediate layers and $\text{tanh}$-activation for the final layers. 

\paragraph{Neural SDE Training Parameters.} We optimize the neural SDE using a signature-kernel score objective (see below, based on \citet{issa2023non}). For training we use the RMSProp optimizer with a learning rate of $0.0001$, $0.0005$ and $0.0005$ for the Cancer, Covid and MIMIC-IV dataset, respectively. We optimize over $15000$ gradient steps, where each gradient step has a batch size of $16$ trajectories for Cancer and $32$ trajectories for Covid and MIMIC-IV. 

\paragraph{Signature Score Loss Function.} Here, we detail how we adapt the signature score training method by \citet{issa2023non} to our conditional training as we have samples from the treatment process and the state process. \citet{issa2023non} propose a non-adversarial training objective for neural SDEs based on the \defined{signature kernel}: The \defined{signature kernel score} of $P \in \mathcal{M}^1(\mathcal{X}, \mathcal{B}(\mathcal{T}_{\mathcal{X}}))$  
and $y \in \mathcal{X}$ is the map $\Phi_{\mathrm{sig}} \colon \mathcal{M}^1(\mathcal{X}, \mathcal{B}(\mathcal{T}_{\mathcal{X}})) \times \mathcal{X} \longrightarrow \mathbb{R}
$
defined by
\begin{align*}
\Phi_{\mathrm{sig}}(P, y) := \mathbb{E}_{x, x' \sim P} \left[ k_{\mathrm{sig}}(x, x') \right] - 2\, \mathbb{E}_{x \sim P} \left[ k_{\mathrm{sig}}(x, y) \right].
\end{align*}
$\Phi_{\mathrm{sig}}$ is a strictly proper score, meaning it assigns the lowest expected score when the proposed prediction is realized by the true probability distribution.
Let $(\Omega, \mathcal{F}, P)$ be a probability space, $T > 0$, $d_x \in \mathbb{N}$. Define path-space $
\mathcal{X} := \left\{ \gamma \in C_{\mathrm{BV}}([0,T], \mathbb{R}^{d_x}) \;\middle|\; \exists i \in [d]:\ \gamma^i \text{ monotone} \right\}$. Let $\mathcal{T}$ be a topology such that the signature kernel $k_{\mathrm{sig}} \colon \mathcal{X} \times \mathcal{X} \to \mathbb{R}$ is continuous. Let $X^{\mathrm{true}} \colon \Omega \to \mathcal{X}$ be a random variable with sample $\{x^{(i)}\}_{i=1}^N \overset{iid}{\sim} P_{X^{\mathrm{true}}} \in \mathcal{M}^1(\mathcal{X}, \mathcal{B}(\mathcal{T}_{\mathcal{X}}))$.

In the unconditional setting (as described in \citet{issa2023non}), one can learn a neural SDE $X^\theta$ into $\mathcal{X}$ such that its law $P_{X^\theta} \approx P_{X^{\mathrm{true}}}$, use the training objective $
\min_\theta \mathcal{L}(\theta), \mathcal{L}(\theta) := \mathbb{E}_{y \sim P_{X^{\mathrm{true}}}} \left[ \Phi_{\mathrm{sig}}(P_{X^\theta}, y) \right]$, 
which is equivalent to optimization using $\mathrm{MMD}^2_{k_{\mathrm{sig}}}[P_{X^\theta}, P_{X^{\mathrm{true}}}]$ as the last term $\mathbb{E}_{y,y' \sim P_{X^{\mathrm{true}}}} [k_{\mathrm{sig}}(y,y')]$ is constant w.r.t.\ $\theta$.

However, our neural SDE is a conditional setting, since we observe both a treatment process and a state process. Given a distribution $P_Z$ to condition on, $P_{X^{\mathrm{true}} \mid Z} (\cdot \mid z)$, $z  \sim P_Z$ the target conditional distribution with sample $\{(z^{(i)}, x^{(i)})\}_{i \in [N]}$, $z^{(i)} \sim P_Z$, $x^{(i)} \sim P_{X^{\mathrm{true}} \mid Z} (\cdot \mid z^{(i)})$. To learn a neural SDE $X^\theta$ into $\mathcal{X}$ such that $P_{X^\theta}(\cdot \mid z) \approx P_{X^{\mathrm{true}}}(\cdot \mid z)$, train with the modified objective:
\begin{align*}
\min_\theta \mathcal{L}'(\theta), \quad \text{where} \quad
\mathcal{L}'(\theta) := \mathbb{E}_{z \sim P_Z} \left[ \mathbb{E}_{x \sim P_{X^{\mathrm{true}}}(\cdot \mid z)} 
\left[ \Phi_{\mathrm{sig}}(P_{X^\theta} (\cdot \mid z), x) \right] \right].
\end{align*}    
Thus, in our controlled SDE setting, we have to adapt the above conditional objective in the following way.
Given a control signal $u$ to condition on, the target conditional distribution $P_{X^{\mathrm{true}} \mid U} (\cdot \mid u)$ with sample $\{(u^{(i)}, x^{(i)})\}_{i \in [N]}$, $u^{(i)} \sim P_U$, $x^{(i)} \sim P_{X^{\mathrm{true}} \mid U} (\cdot \mid u^{(i)})$. To learn a neural SDE $X^\theta$ into $\mathcal{X}$ such that $P_{X^\theta}(\cdot \mid u,x_0) \approx P_{X^{\mathrm{true}}}(\cdot \mid u,x_0)$, train with the modified objective:
$$
\min_\theta \mathcal{L}'(\theta), \quad \text{where} \quad
\mathcal{L}'(\theta) := \mathbb{E}_{u \sim P_U, x_0 \sim P_{X_0}} \left[ \mathbb{E}_{x \sim P_{X^{\mathrm{true}}}(\cdot \mid u, x_0)} 
\left[ \Phi_{\mathrm{sig}}(P_{X^\theta} (\cdot \mid u, x_0), x) \right] \right].
$$

In all of our experiments, we use an RBF-lifting kernel with a smoothing parameter $\sigma = 1$ to compute the above loss. We set the order of dyadic refinement associated to the PDE solver for computing the signature kernel to $1$. Following \citet{issa2023non}, we apply a time-normalization transform and augment true and predicted paths with the normalized-time component before passing them through the loss function.

\subsection*{\ours regularizer}

We compute the Signature-Kernel MCMD regularizer using the held-out validation dataset. Specifically, the validation sample are $\{x^{(i)},u^{(i)},x_0^{(i)}\}_{i \in [N]}$ and we use our trained neural SDE to obtain  $\{\hat{x}^{(i)}, u^{(i)}, x_0^{(i)}\}_{i \in [N]}$ using the same initial states and controls. We use a linear lifting kernel with smoothing parameter $\sigma = 1$ for the trajectories. Again, we use a dyadic refinement factor $1$ for the PDE solver and apply a time-normalization transform and augment true and predicted paths with the normalized-time component before evaluating the path. For the initial states, we use an RBF-kernel with smoothing parameter $\sigma = 1$.

\subsection*{Libraries}

For all signature-kernel implementations (for the regularizer) or for training the neural SDE, we build on the following JAX implementation: \url{https://github.com/crispitagorico/sigkerax}.

\subsection*{Treatment Optimization}

We solve the treatment optimization problem using a direct single-shooting approach. For details on numeric approaches to solving continuous-time optimal control problems, we refer to \cite{gros2020numerical}. 

\subsubsection*{Treatment Parameterization} 

We parameterize the control function $u(t)$ (denoted by $u(t,\phi)$) with the parameter $\phi$. Specifically, for the Cancer tasks where we have a two-dimensional control consisting of the chemo- and radiotherapy treatment, we parameterize the chemo parameters $\phi_{c} = [\phi_{c,timepoints}, \phi_{c,dosage}]$, with  $\phi_{c,timepoints}$ denoting the timepoints of treatment administration, and $\phi_{c, dosage}$ the respective dosages. Similarly we parameterize the radiotherapy parameters $\phi_{r}$.  We restrict the controls to a maximum of $K$ treatment administrations for chemo- and radiotherapy, respectively. Now, to obtain a control function from our parameters, we model the control similarly to the data generating process and model the control function $u_{chemo}(t)$ by modeling a bang-bang control and setting $u_{chemo}(t) = \sum_{i=1}^{K} \mathbb{I}_{(t \in [\phi_{c,timepoints, i}, \phi_{c,timepoints, i} + 1])} \cdot \phi_{c,dosage; i}$, where $\mathbb{I}$ is the indicator function. The function $u_{radio}(t)$ is constructed similarly.

For the COVID-19 task, we have a one-dimensional control corresponding to the Dexamethasone administration. As parameters, we directly parameterize the time series of Dexamethasone administration with $42$ timesteps, and the continuous control function is given by the linear interpolation between the $42$ parameters.

\subsubsection*{Treatment Optimization}

We use a parameterized representation of the control functions (as described above). This allows for direct single-shooting control optimization. For cancer, we initialize the timepoint parameters $\phi_{timepoints}$ by uniformly sampling over the time interval $I$. The dosage parameters $\phi_{dosage}$ are initialized by uniformly sampling between $[0.1, 0.3]$. For Covid-19, we initialize the timeseries parameters by uniformly sampling between $[0.0, 1.0]$.  We use gradient-descent-based methods (RMSProp optimizer) with a learning rate of $0.001$ and $5000$ gradient-update steps. Moreover, we use Monte Carlo sampling with a sample size $=100$ to approximate the expectation in our objective function.
\section{Baseline Implementation} \label{app:baseline_implementation}

\subsection{Control Library Protocol}

The library of candidate controls consists of $100$ randomly sampled control candidates. Importantly, every method evaluates the exact same set of candidate controls. We randomly sample the parameters $\phi$, which parameterize the continuous-time control function. For the respective problems, we use the parameterization described in \cref{app:practical_implementation}. For cancer, we initialize the treatment timepoints $\phi_{timepoints}$ by uniformly sampling over the time interval of interest $[t_0, t_f]$. The dosages are obtained by randomly sampling over the interval $[0.1, 0.3]$ (in the min-max-normalized control space). For the Cancer tasks, we choose a maximum of $K=5$ treatment administrations. For Covid-19, we randomly sample the control for each of the $42$ timepoints in the time-interval randomly from $[0,1]$ (in the min-max-normalized control space). Each method, including \ourslibrary, predicts the patient outcome for each candidate control and selects the control with the lowest predicted costs as optimal.

\subsection{Evaluation Metrics} \label{app:evaluation_metrics}

\paragraph{Control-Library Metrics.} For each method, we report the Spearman rank correlation between the predicted costs and the true costs. Let $\mathcal{J}$ be the cost function (see \Cref{app:datasets} for our choices). Moreover, let $\{u_i\}_{i=1}{M}$ be the control library. Each method predicts the costs for the control candidates, $\{\hat{\mathcal{J}}_i\}_{i=1}^{M}$, and we can use the simulators to obtain the true costs for the candidates $\{\mathcal{J}_i\}_{i=1}^{M}$. The Spearman rank correlation is then computed between  $\{\hat{\mathcal{J}}_i\}_{i=1}^{M}$ and $\{\mathcal{J}_i\}_{i=1}^{M}$. Thus, higher values indicate that a method keeps the same ordering among candidate controls. Further, we report the ground-truth costs directly, i.e., $\mathcal{J}$, for the control from the control library that a method predicts as optimal. 

\paragraph{Ground-Truth Costs.} For an optimized treatment plan $u$, the ground-truth costs are $\mathcal{J}(t_0, X^{t_0,x_0,u},u)$, where the expectations of $X$ are computed using MC sampling using the true, underlying SDE. Moreover, we have $\sqrt{\text{Ground-Truth Cost}} = \sqrt{\mathcal{J}(t_0, X^{t_0,x_0,u},u)}$. 

\paragraph{Relative Ground-Truth Costs.} The relative ground-truth costs aim at evaluating how well a method performs compared to naive optimization. Specifically, let $u_{\text{naive}}$ be the treatment plan obtained via naive optimization ($\lambda=0$) with corresponding ground truth-costs $\mathcal{J}(t_0, X^{t_0,x_0,u_{\text{naive}}},u_{\text{naive})}$ and let $u_{\text{method}}$ be the treatment plan obtained by a method (e.g., choosing $\lambda > 0$ in \ours) with corresponding ground-truth costs  $\mathcal{J}(t_0, X^{t_0,x_0,u_{\text{method}}},u_{\text{method})}$. The relative ground-truth costs are defined as: 

\begin{align*}
\frac{\mathcal{J}(t_0, X^{t_0,x_0,u_{\text{method}}},u_{\text{method})}}{\mathcal{J}(t_0, X^{t_0,x_0,u_{\text{naive}}},u_{\text{naive})}}.
\end{align*}

Thus, if the relative ground-truth costs are $1$ both treatments perform equally. Values $<1$ suggest improvement by $u_{\text{method}}$ and values $>1$ suggest worse performance of $u_{\text{method}}$ compared to naively optimized treatment plans. We evaluate the expectations in the ground-truth costs using a monte carlo sample size of $200$.

\subsection{Baselines}
\paragraph{TE-CDE \cite{seedat2022continuous}.} 
In our experiments, we use the same model architecture as proposed in the original paper, meaning a 2-layer neural network with hidden state dimension $128$ for both the neural CDE encoder and decoder. We train the decoder already to predict the same time horizon as needed in our control tasks. We also adjust the balancing-representation loss objective from a cross-entropy loss to a mean-squared-error loss as it naturally extends to our general formulation of continuous-valued treatments. We perform a grid-search for selecting the learning rate $\in \{0.0001, 0.001, 0.01\}$, number of epochs $\in \{50, 100\}$ and dimensions of the latent state $\in \{8, 16\}$. The optimal hyperparameters are determined by choosing the set that results in the lowest mean squared error on the validation dataset, for the Cancer and Covid datasets, respectively. For TE-CDE we build on the implementation provided in \url{https://github.com/konstantinhess/Bayesian-Neural-CDE.git}.

\paragraph{INSITE/SINDy \cite{kacprzyk2024ode}.} INSITE is a framework for bridging treatment effect estimation and dynamical system discovery. Following \citet{kacprzyk2024ode}, we also instantiate the INSITE framework using SINDy \citep{brunton2016discovering}. Importantly, since we do not have patient heterogeneity in our experiments (so where each patient has their ``own'' dynamics parameters) and instead model patient-conditioned effects by conditioning on the observed history, we only obtain the population-level dynamics equation and not further fine-tune it on a specific trajectory. Moreover following \cite{kacprzyk2024ode}, we incorporate the continuous-valued treatment directly as a control input and use SINDy with control. %
In order to establish a fair comparison with the SINDy baseline, similar to \citet{kacprzyk2024ode}, we used the STLSQ optimizer together with the polynomial library and performed hyperparameter-tuning over the polynomial degree $\mathrm{deg} \in \lbrace 1,2, 3, 4 \rbrace$, threshold $\lbrace 0.1,0.2,0.5 \rbrace$ and the ridge regularization parameter $\lambda \in \lbrace 0.1,0.2,0.5 , 1, 2, 3\rbrace$. We implement INSITE based on the \href{https://github.com/samholt/ODE-Discovery-for-Longitudinal-Heterogeneous-Treatment-Effects-Inference}{original repository}.

\paragraph{CT-MOPO} %
Among model-based offline RL baselines, which are built for discrete-time dynamics and closed-loop policies, making them poorly suited to irregular continuous-time open-loop treatment optimization, MOPO (\citep{yu2020mopo} ) is the most natural candidate for a reasonable continuous-time adaptation. To obtain the strongest possible MOPO-inspired baseline, which is as close as possible to the original pessimistic design, we use the same neural-SDE backbone and gradient-based open-loop optimization pipeline as our method. More precisely we
\begin{enumerate}
    \item[(i)] replace the closed-loop policy by a trainable open loop treatment path $u_\phi:[t_0,t_f]\to\mathbb{R}^{d_u}$, trained via gradient descent
    \item[(ii)] replace the probabilistic ensemble of discrete-time dynamics models by an ensemble of $M=5$ independently initialized and trained neural SDEs with drift and diffusion functions $(\mu_{\theta_m},\sigma_{\theta_m})$
\end{enumerate}
and after each ensemble member is rolled out with control $u_\phi$ to give state trajectory $\widehat X^{m,u_\phi}$ we
\begin{enumerate}
    \item[(iii)] compute the covariance at each time $t$ for each member $m$
    \begin{align*}
    \Sigma_m(t,u_\phi)
    = \sigma_{\theta_m}(\widehat X_t^{m,u_\phi},u_\phi(t))
      \sigma_{\theta_m}(\widehat X_t^{m,u_\phi},u_\phi(t))^\top .
    \end{align*}
    \item[(iv)] calculate the conservative uncertainty signal across ensemble members (for diagonal diffusions) via 
    \begin{align*}
    r(t,u_\phi)
    = \max_{m=1,\dots,M}
      \sqrt{\operatorname{tr}(\Sigma_m(t,u_\phi))}. %
    \end{align*}
    Note that in the original MOPO implementation, a diagonal Gaussian output model is used. 
\end{enumerate}
The resulting open-loop MOPO objective is
\begin{align*}
    \min_\phi\; \widehat{\mathcal J}(u_\phi)
    + \lambda \int_{t_0}^{t_f} r(t,u_\phi)\,\mathrm{d}t,
\end{align*}
where $\widehat{\mathcal J}(u_\phi)$ denotes the cost evaluated under the learned neural-SDE rollout.

\section{Experiments Compute Resources}\label{app:compute_resources}
All experiments were conducted on a high-performance computing (HPC) cluster equipped with NVIDIA A100 and V100 GPUs and Intel Xeon Gold 6142 CPUs. Each model training run, including those for our method and the baselines, was run with a wall-clock time limit of 4 hours on a single GPU. Each optimization-task initialization was also run with a wall-clock time limit of 4 hours.
\section{Additional Experimental Results}\label{app:additional_experimental_results}

\subsection{Additional Evaluation Metrics for \ours}

In \Cref{fig:mmd_and_mcmd} \textbf{A}, we report the squared signature-kernel MMD between the true and predicted trajectories under optimized treatment plans, relative to naive optimization. For increasing regularization strengths, the relative MMD decreases, suggesting that the predicted and true trajectories are more aligned than in the naive optimization. In \Cref{fig:mmd_and_mcmd} \textbf{B}, we show the value of our \ours regularizer relative to the value obtained using naive optimization. The ratio is consistently below 1.

\begin{figure}[htbp]
  \centering
    \includegraphics[width=\textwidth]{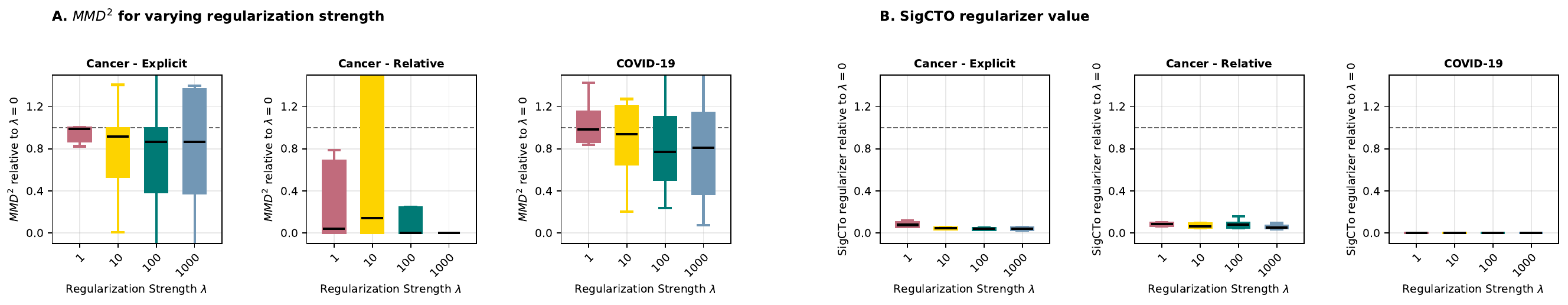}
    \caption{\textbf{A:} $\mathrm{MMD^2}_{k_{\mathrm{sig}}}$ between the true and predicted patient trajectory distributions for the optimal treatment plan obtained with different levels of regularization strength $\lambda$. \textbf{B:} $\mathrm{MCMD}_{k_{\mathrm{sig}}}$ estimator for the optimal control obtained with different levels of regularization strength $\lambda$. Results are reported using the rbf kernel and $K=1$. Boxplots summarize results over $15$ task initializations; boxes show medians and interquartile ranges (IQR), and whiskers extend to the farthest point within $1.5 \times \mathrm{IQR}$.}
    \label{fig:mmd_and_mcmd}
\end{figure}

\subsection{Ablations}
\Cref{fig:linear_kernel} shows an ablation in which we replace the linear kernel in the \ours signature kernel MMD regularizer with an RBF kernel. As with the linear kernel, increasing the regularization strength $\lambda$ generally reduces relative ground-truth cost. However, increasing $\lambda$ from $100$ to $1000$ raises the relative ground-truth cost on both cancer tasks. This suggests that excessive regularization may dominate the treatment-cost term during optimization and limit further improvements in the treatment optimization task. Increasing the number of time intervals $K$ results in greater improvements with the RBF kernel than with the linear kernel on both cancer tasks, whereas there is no consistent improvement for the Covid-19 task.
\begin{figure}[htbp]
  \centering
    \includegraphics[width=\linewidth]{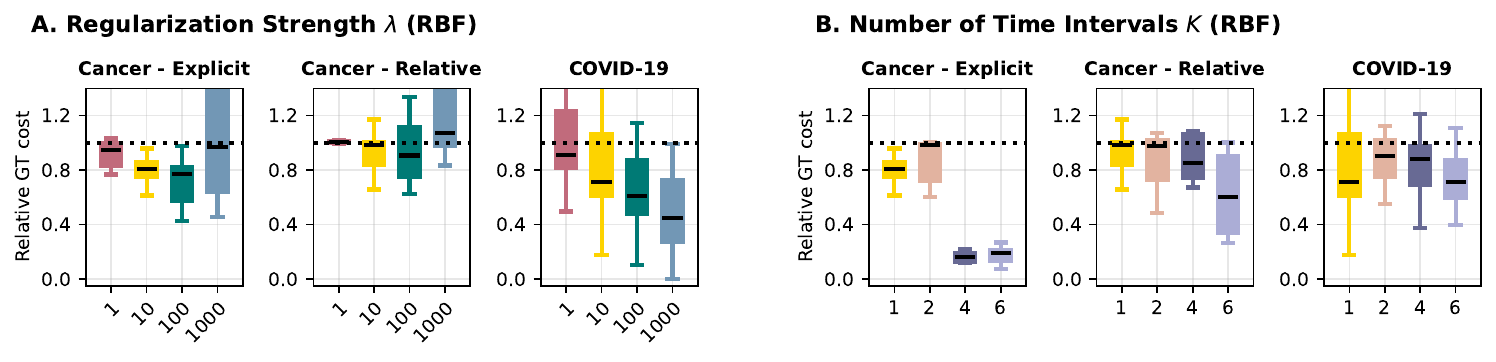}
    \caption{\textbf{\ours on PK/PD simulation tasks using an RBF kernel in the signature kernel MMD regularizer.} \textbf{A:} \& \textbf{B.} Relative ground-truth cost for varying regularization strength $\lambda$ and number of time intervals $K$, respectively. Values below $1$ indicate improvement over naive optimization. Boxplots summarize results over $15$ task initializations, boxes show medians and interquartile ranges (IQR), and whiskers extend to the farthest point within $1.5 \times \mathrm{IQR}$.}
    \label{fig:linear_kernel}
\end{figure}

\subsection{Sensitivity to violations of structural model assumption}
In \Cref{tab:sensitivity}, we evaluate the sensitivity of \ours on the Cancer-Explicit task to variations in the data-generating dynamics and violations of structural model assumptions. Specifically, \Cref{tab:structural_violations} considers two modifications to the cancer simulator (\Cref{app:datasets}): (1) setting the diffusion coefficients to zero to obtain an ODE model and (2) withholding the chemotherapy concentration from the observed state, resulting in partially observed, non-Markovian dynamics. The regularizer remains effective in the ODE setting but substantially increases relative ground-truth costs in the partially observed setting. A potential way to address this limitation is to augment the state with past observations and treatments to better capture history dependence.

We also assess sensitivity to time-dependent confounding in treatment assignment by increasing its strength $\gamma$ (\Cref{tab:overlap}), using a setup similar to the study of overlap violations in \cite{kacprzyk2024ode}. Following the setup of \cite{kacprzyk2024ode}, we use a different treatment assignment process for this experiment. Consequently, the results at $\gamma=0$ are not directly comparable to those in \Cref{tab:structural_violations}. Nevertheless, \ours remains effective in this setting, reducing the mean ground-truth cost relative to naive optimization. Compared with the unconfounded setting, both confounded settings yield higher ground-truth costs, suggesting that treatment optimization becomes more challenging. Nevertheless, the mean relative ground-truth costs remain well below $1$ at $\gamma=0$ and $\gamma=1$, demonstrating a continued benefit from our \ours regularization to prevent model exploitation. At $\gamma=2$, the mean relative cost rises to $1.84 \pm 3.96$, while the median remains low at $0.25$. Thus, regularization still improves performance in most task initializations under strong confounding, although substantial cost increases in some cases outweigh these gains on average.

Finally, we assess sensitivity to correlated process noise by introducing nonzero off-diagonal entries in the simulator's diffusion matrix. At the strongest tested noise coupling, $\sigma=0.4$, the ground-truth costs are higher than in the baseline setting, suggesting increased difficulty of the treatment optimization task. Despite this increase, the mean relative ground-truth costs remain well below $1$ across all tested noise levels: $0.15$, $0.42$, and $0.40$ for $\sigma=0$, $0.1$, and $0.4$, respectively. These results show that the regularizer retains a substantial average benefit over naive optimization even under correlated process noise, consistent with its role in limiting model exploitation.

\begin{table}[t]
    \centering
    \small
    \caption{Sensitivity results on the Cancer-Explicit task. We evaluate $\ours$ using a Linear kernel with $\lambda = 100$ and $K=1$.All values are reported as mean ± standard deviation over $15$ task initializations.}
    \label{tab:sensitivity}

    \begin{subtable}[t]{\textwidth}
        \centering
        \caption{Sensitivity to Structural Variations}
        \label{tab:structural_violations}
        \begin{tabular}{lccc}
            \toprule
            & \textbf{SDE} & \textbf{ODE} & \textbf{Non-Markovian SDE} \\
            \midrule
            $\sqrt{\text{Ground Truth Cost}} (\downarrow)$
                & $9.91\pm13.66$
                & $1.45\pm2.43$
                & $2.81\pm2.21$ \\
            Relative ground-truth costs ($\downarrow$)
                & $0.15\pm0.20$
                & $0.06\pm0.13$
                & $10.11\pm9.00$ \\
            \bottomrule
        \end{tabular}
    \end{subtable}

    \vspace{1em}

    \begin{subtable}[t]{\textwidth}
        \centering
        \caption{Overlap violations/ Confounded Treatment}
        \label{tab:overlap}
        \begin{tabular}{lccc}
            \toprule
            & \textbf{$\gamma = 0 $} & \textbf{$\gamma = 1 $} & \textbf{$\gamma = 2 $} \\
            \midrule
            $\sqrt{\text{Ground Truth Cost}} (\downarrow)$
                & $0.57\pm0.01$
                & $5.15\pm4.95$
                &  $4.57\pm7.07$ \\
            Relative ground-truth costs ($\downarrow$)
                & $0.03\pm0.03$
                & $0.42\pm0.46$
                & $1.84\pm3.96$ \\
            \bottomrule
        \end{tabular}
    \end{subtable}
    \vspace{1em}

    \begin{subtable}[t]{\textwidth}
        \centering
        \caption{Confounded Noise}
        \label{tab:hidden-confounding}
        \begin{tabular}{lccc}
            \toprule
            & \textbf{$\sigma = 0$} & \textbf{$\sigma = 0.1$} & \textbf{$\sigma = 0.4$} \\
            \midrule
            $\sqrt{\text{GT-Cost}}$
                & $9.91\pm13.66$
                & $6.91\pm16.40$
                & $26.18\pm25.93$ \\
            Relative ground-truth costs ($\downarrow$) 
                &$0.15\pm0.20$
                & $0.42\pm1.24$ 
                & $0.40\pm0.42$\\
            \bottomrule
        \end{tabular}
    \end{subtable}

\end{table}

\end{document}